\documentclass[twoside]{article}

\usepackage[accepted]{aistats2012}

\usepackage{times}
\usepackage{url}
\usepackage[colorlinks, citecolor=blue,linkcolor=blue]{hyperref}
\usepackage{graphicx}
\usepackage{amsmath}

\usepackage{algorithm}
\usepackage{algpseudocode}

\usepackage{tikz}
\usepackage{pgfplots}
\usetikzlibrary{decorations.pathreplacing}
\usepgfplotslibrary{groupplots}
\usetikzlibrary{plotmarks}

\usepackage{natbib}
\bibpunct{(}{)}{;}{z}{}{,}

\input{Definitions}

\begin{document}

%
\runningtitle{Quilting Stochastic Kronecker Product Graphs to Generate
  Multiplicative Attribute Graphs}

%

\twocolumn[

\aistatstitle{Quilting Stochastic Kronecker Product Graphs to Generate
  \\ Multiplicative Attribute Graphs}

\aistatsauthor{ Hyokun Yun \And S.V$\!.\,$N. Vishwanathan}

\aistatsaddress{ Department of Statistics\\Purdue University \And
  Department of Statistics and Computer Science\\Purdue University } ]

\begin{abstract}
  We describe the first sub-quadratic sampling algorithm for
  the Multiplicative Attribute Graph Model (MAGM) of \citet{KimLes10}.  We
  exploit the close connection between MAGM and the Kronecker Product
  Graph Model (KPGM) of \citet{LesChaKleFaletal10}, and show that to
  sample a graph from a MAGM it suffices to sample small number of KPGM
  graphs and \emph{quilt} them together. Under a restricted set of 
  technical conditions
  our algorithm runs in $O\rbr{(\log_2(n))^3 \abr{E}}$ time, where $n$
  is the number of nodes and $\abr{E}$ is the number of edges in the
  sampled graph. We demonstrate the scalability of our algorithm via
  extensive empirical evaluation; we can sample a MAGM graph with 8
  million nodes and 20 billion edges in under 6 hours. 


\end{abstract}

\section{Introduction}
\label{sec:Introduction}

In this paper we are concerned with statistical models on
graphs. Typically one is interested in two aspects of graph models,
namely scalable inference and efficiency in generating samples. While
scalable inference is very important, an efficient sampling algorithm is
also critical:
\begin{itemize}
\item To assess the goodness of fit, one generates graphs from the model
  and compares graph statistics of the samples with the original graph
  \citep{HunGooHan08}.
\item To test whether a certain motif is overrepresented in the graph,
  one strategy is to sample large number of graphs in the null
  hypothesis to approximate the $p$-value of the test statistic
  \citep{SheMilManAlo02}.
\item To predict the future growth of the graph, one may fit the model
  on the current graph and generate a larger graph with the estimated
  parameters.
\end{itemize}

The stochastic Kronecker Product Graph Model (KPGM) was introduced by
\citet{LesChaKleFaletal10} as a scalable statistical model on graphs.
Compared to previous models such as Exponential Random Graph Models
(ERGMs) \citep{RobPatKalLus07} or latent factor models \citep{Hoff09},
KPGM sports a number of advantages. In particular, the inference
algorithm of KPGM is scalable to very large graphs, and sampling a graph
from the model takes time that is proportional to the expected number of
edges. This makes the KPGM a very attractive model to study. However,
\citet{MorNev09} report that KPGM fails to capture some characteristics
of real-world graphs, such as power-law degree distribution and local
clustering.

In order to address some of the above shortcomings and to generalize the
expressiveness of KPGM, \citet{KimLes10} recently proposed the
Multiplicative Attribute Graph Model (MAGM).  MAGM can provably model
the power-law degree distribution while KPGM can not. Furthermore,
\citet{KimLes10} demonstrate empirically that MAGM is better able to
capture graph statistics of real-world graphs \citep{KimLes11}.  The
issue of inference for MAGM is addressed by \citet{KimLes11} via an
efficient variational EM algorithm.

However, the issue of \emph{efficiently} sampling graphs from MAGM
remains open. Currently, all algorithms that we are aware of scale as
$O(n^2)$ in the worst case, where $n$ is the number of nodes. This is
because, the probability of observing an edge between two nodes is
determined by the so-called $n \times n$ edge probability matrix. Unlike
the case of KPGM where the edge probability matrix has a Kronecker
structure which can be exploited to efficiently sample graphs in
expected $O(\log_{2}(n) |E|)$ time, where $|E|$ is the number of
edges, no such results are known for MAGM.  In the absence of any
structure, naively sampling each entry of the adjacency matrix requires
$O(n^2)$ Bernoulli trials, which is prohibitively expensive for
generating real-world graphs with millions of nodes.


In this paper we show that under a restricted set of technical
conditions, with high
probability, a significant portion of the edge probability matrix of
MAGMs is the same as that of KPGMs (modulo permutations).  We then
exploit this observation to \emph{quilt} $O((\log_{2}(n))^2)$ graphs
sampled from a KPGM to form a single sample from a MAGM.  The expected
time complexity of our sampling scheme is thus $O( (\log_{2}(n))^3 
|E| )$.


\subsection{Notation and Preliminaries}
\label{sec:Notation}

A graph $G$ consists of an ordered set of $n$ nodes
$V=\{1,2, \ldots, n\}$, and a set of directed edges $E \subset V \!\times\!
V$. A node $i$ is said to be a neighbor of another node $j$ if
they are connected by an edge, that is, if $(i,j)\in E$.
Furthermore, for each edge $(i,j)$, $i$ is called the source node of the
edge, and $j$ the target node.  We define the adjacency matrix of graph
$G$ as the $n\times n$ matrix $A$ with $A_{ij} = 1$ if $(i,j) \in E$,
and $0$ otherwise.  Also, the following standard definition of Kronecker
product is used \citep{Bernstein05}:
\begin{definition}
  \label{def:kron}
  Given real matrices $X \in \RR^{n \times m}$ and $Y \in \RR^{p \times
    q}$, the Kronecker product $X \otimes Y \in \RR^{np \times mq}$ is
  \begin{align*}
    X \otimes Y := \mymatrix{cccc}{X_{11} Y & X_{12} Y & \ldots & X_{1m} Y \\
      \vdots & \vdots & \vdots & \vdots \\
      X_{n1} Y & X_{n2} Y & \ldots & X_{nm} Y}.
  \end{align*}
  The $k$-th Kronecker power $X^{\sbr{k}}$ is $\otimes_{i=1}^{k} X$.
\end{definition}

\section{Kronecker Product Graph Model}
\label{sec:kpgm}

The Stochastic Kronecker Product Graph Model of
\citet{LesChaKleFaletal10} is usually parametrized by a $2 \times 2$
initiator matrix
\begin{align}
  \label{eq:theta-def}
  \Theta :=
  \mymatrix{cc}{
    \theta_{00} & \theta_{01} \\
    \theta_{10} & \theta_{11}
  },
\end{align}
with each $\theta_{ij} \in [0, 1]$, and a size parameter $d \in
\mathbb{Z}^+$. For simplicity, let the number of nodes $n$ be $2^d$.
The case where $n < 2^d$ can be taken care by discarding $(2^d - n)$
number of nodes later as discussed in \citet{LesChaKleFaletal10},
but in our context we will only use $n = 2^d$ for
KPGM.

The $n \times n$ edge probability matrix $P$ is defined as the $d$-th
Kronecker power of the parameter $\Theta$, that is,
\begin{align}
  \label{eq:kpgm_p1}
  P = \Theta^{\sbr{d}} = \underbrace{\Theta \otimes \Theta \otimes \ldots
    \otimes \Theta}_{d \text{ times}}.
\end{align}
The probability of observing an edge between node $i$ and $j$ is simply
the $(i,j)$-th entry of $P$, henceforth denoted as $P_{ij}$.  See
Figure~\ref{fig:ep_matrix} (left) for an example of the edge probability
matrix of a KPGM, and observe its fractal structure which follows from
the definition of $P$.

\begin{figure}[h]
  \centering
  \begin{tikzpicture}[scale=1.3]
    \draw[draw, fill=black!6] (0.000000,2.400000) rectangle (0.300000, 2.100000);
\draw[draw, fill=black!11] (0.000000,2.100000) rectangle (0.300000, 1.800000);
\draw[draw, fill=black!11] (0.000000,1.800000) rectangle (0.300000, 1.500000);
\draw[draw, fill=black!19] (0.000000,1.500000) rectangle (0.300000, 1.200000);
\draw[draw, fill=black!11] (0.000000,1.200000) rectangle (0.300000, 0.900000);
\draw[draw, fill=black!19] (0.000000,0.900000) rectangle (0.300000, 0.600000);
\draw[draw, fill=black!19] (0.000000,0.600000) rectangle (0.300000, 0.300000);
\draw[draw, fill=black!34] (0.000000,0.300000) rectangle (0.300000, 0.000000);
\draw[draw, fill=black!11] (0.300000,2.400000) rectangle (0.600000, 2.100000);
\draw[draw, fill=black!14] (0.300000,2.100000) rectangle (0.600000, 1.800000);
\draw[draw, fill=black!19] (0.300000,1.800000) rectangle (0.600000, 1.500000);
\draw[draw, fill=black!25] (0.300000,1.500000) rectangle (0.600000, 1.200000);
\draw[draw, fill=black!19] (0.300000,1.200000) rectangle (0.600000, 0.900000);
\draw[draw, fill=black!25] (0.300000,0.900000) rectangle (0.600000, 0.600000);
\draw[draw, fill=black!34] (0.300000,0.600000) rectangle (0.600000, 0.300000);
\draw[draw, fill=black!44] (0.300000,0.300000) rectangle (0.600000, 0.000000);
\draw[draw, fill=black!11] (0.600000,2.400000) rectangle (0.900000, 2.100000);
\draw[draw, fill=black!19] (0.600000,2.100000) rectangle (0.900000, 1.800000);
\draw[draw, fill=black!14] (0.600000,1.800000) rectangle (0.900000, 1.500000);
\draw[draw, fill=black!25] (0.600000,1.500000) rectangle (0.900000, 1.200000);
\draw[draw, fill=black!19] (0.600000,1.200000) rectangle (0.900000, 0.900000);
\draw[draw, fill=black!34] (0.600000,0.900000) rectangle (0.900000, 0.600000);
\draw[draw, fill=black!25] (0.600000,0.600000) rectangle (0.900000, 0.300000);
\draw[draw, fill=black!44] (0.600000,0.300000) rectangle (0.900000, 0.000000);
\draw[draw, fill=black!19] (0.900000,2.400000) rectangle (1.200000, 2.100000);
\draw[draw, fill=black!25] (0.900000,2.100000) rectangle (1.200000, 1.800000);
\draw[draw, fill=black!25] (0.900000,1.800000) rectangle (1.200000, 1.500000);
\draw[draw, fill=black!32] (0.900000,1.500000) rectangle (1.200000, 1.200000);
\draw[draw, fill=black!34] (0.900000,1.200000) rectangle (1.200000, 0.900000);
\draw[draw, fill=black!44] (0.900000,0.900000) rectangle (1.200000, 0.600000);
\draw[draw, fill=black!44] (0.900000,0.600000) rectangle (1.200000, 0.300000);
\draw[draw, fill=black!56] (0.900000,0.300000) rectangle (1.200000, 0.000000);
\draw[draw, fill=black!11] (1.200000,2.400000) rectangle (1.500000, 2.100000);
\draw[draw, fill=black!19] (1.200000,2.100000) rectangle (1.500000, 1.800000);
\draw[draw, fill=black!19] (1.200000,1.800000) rectangle (1.500000, 1.500000);
\draw[draw, fill=black!34] (1.200000,1.500000) rectangle (1.500000, 1.200000);
\draw[draw, fill=black!14] (1.200000,1.200000) rectangle (1.500000, 0.900000);
\draw[draw, fill=black!25] (1.200000,0.900000) rectangle (1.500000, 0.600000);
\draw[draw, fill=black!25] (1.200000,0.600000) rectangle (1.500000, 0.300000);
\draw[draw, fill=black!44] (1.200000,0.300000) rectangle (1.500000, 0.000000);
\draw[draw, fill=black!19] (1.500000,2.400000) rectangle (1.800000, 2.100000);
\draw[draw, fill=black!25] (1.500000,2.100000) rectangle (1.800000, 1.800000);
\draw[draw, fill=black!34] (1.500000,1.800000) rectangle (1.800000, 1.500000);
\draw[draw, fill=black!44] (1.500000,1.500000) rectangle (1.800000, 1.200000);
\draw[draw, fill=black!25] (1.500000,1.200000) rectangle (1.800000, 0.900000);
\draw[draw, fill=black!32] (1.500000,0.900000) rectangle (1.800000, 0.600000);
\draw[draw, fill=black!44] (1.500000,0.600000) rectangle (1.800000, 0.300000);
\draw[draw, fill=black!56] (1.500000,0.300000) rectangle (1.800000, 0.000000);
\draw[draw, fill=black!19] (1.800000,2.400000) rectangle (2.100000, 2.100000);
\draw[draw, fill=black!34] (1.800000,2.100000) rectangle (2.100000, 1.800000);
\draw[draw, fill=black!25] (1.800000,1.800000) rectangle (2.100000, 1.500000);
\draw[draw, fill=black!44] (1.800000,1.500000) rectangle (2.100000, 1.200000);
\draw[draw, fill=black!25] (1.800000,1.200000) rectangle (2.100000, 0.900000);
\draw[draw, fill=black!44] (1.800000,0.900000) rectangle (2.100000, 0.600000);
\draw[draw, fill=black!32] (1.800000,0.600000) rectangle (2.100000, 0.300000);
\draw[draw, fill=black!56] (1.800000,0.300000) rectangle (2.100000, 0.000000);
\draw[draw, fill=black!34] (2.100000,2.400000) rectangle (2.400000, 2.100000);
\draw[draw, fill=black!44] (2.100000,2.100000) rectangle (2.400000, 1.800000);
\draw[draw, fill=black!44] (2.100000,1.800000) rectangle (2.400000, 1.500000);
\draw[draw, fill=black!56] (2.100000,1.500000) rectangle (2.400000, 1.200000);
\draw[draw, fill=black!44] (2.100000,1.200000) rectangle (2.400000, 0.900000);
\draw[draw, fill=black!56] (2.100000,0.900000) rectangle (2.400000, 0.600000);
\draw[draw, fill=black!56] (2.100000,0.600000) rectangle (2.400000, 0.300000);
\draw[draw, fill=black!72] (2.100000,0.300000) rectangle (2.400000, 0.000000);
  \end{tikzpicture}
  \;\;\;
  \begin{tikzpicture}[scale=1.3]
    \draw[draw, fill=black!32] (0.000000,2.400000) rectangle (0.300000, 2.100000);
\draw[draw, fill=black!44] (0.000000,2.100000) rectangle (0.300000, 1.800000);
\draw[draw, fill=black!44] (0.000000,1.800000) rectangle (0.300000, 1.500000);
\draw[draw, fill=black!34] (0.000000,1.500000) rectangle (0.300000, 1.200000);
\draw[draw, fill=black!34] (0.000000,1.200000) rectangle (0.300000, 0.900000);
\draw[draw, fill=black!19] (0.000000,0.900000) rectangle (0.300000, 0.600000);
\draw[draw, fill=black!25] (0.000000,0.600000) rectangle (0.300000, 0.300000);
\draw[draw, fill=black!25] (0.000000,0.300000) rectangle (0.300000, 0.000000);
\draw[draw, fill=black!44] (0.300000,2.400000) rectangle (0.600000, 2.100000);
\draw[draw, fill=black!32] (0.300000,2.100000) rectangle (0.600000, 1.800000);
\draw[draw, fill=black!44] (0.300000,1.800000) rectangle (0.600000, 1.500000);
\draw[draw, fill=black!25] (0.300000,1.500000) rectangle (0.600000, 1.200000);
\draw[draw, fill=black!25] (0.300000,1.200000) rectangle (0.600000, 0.900000);
\draw[draw, fill=black!19] (0.300000,0.900000) rectangle (0.600000, 0.600000);
\draw[draw, fill=black!34] (0.300000,0.600000) rectangle (0.600000, 0.300000);
\draw[draw, fill=black!34] (0.300000,0.300000) rectangle (0.600000, 0.000000);
\draw[draw, fill=black!44] (0.600000,2.400000) rectangle (0.900000, 2.100000);
\draw[draw, fill=black!44] (0.600000,2.100000) rectangle (0.900000, 1.800000);
\draw[draw, fill=black!32] (0.600000,1.800000) rectangle (0.900000, 1.500000);
\draw[draw, fill=black!25] (0.600000,1.500000) rectangle (0.900000, 1.200000);
\draw[draw, fill=black!25] (0.600000,1.200000) rectangle (0.900000, 0.900000);
\draw[draw, fill=black!19] (0.600000,0.900000) rectangle (0.900000, 0.600000);
\draw[draw, fill=black!25] (0.600000,0.600000) rectangle (0.900000, 0.300000);
\draw[draw, fill=black!25] (0.600000,0.300000) rectangle (0.900000, 0.000000);
\draw[draw, fill=black!34] (0.900000,2.400000) rectangle (1.200000, 2.100000);
\draw[draw, fill=black!25] (0.900000,2.100000) rectangle (1.200000, 1.800000);
\draw[draw, fill=black!25] (0.900000,1.800000) rectangle (1.200000, 1.500000);
\draw[draw, fill=black!14] (0.900000,1.500000) rectangle (1.200000, 1.200000);
\draw[draw, fill=black!14] (0.900000,1.200000) rectangle (1.200000, 0.900000);
\draw[draw, fill=black!11] (0.900000,0.900000) rectangle (1.200000, 0.600000);
\draw[draw, fill=black!19] (0.900000,0.600000) rectangle (1.200000, 0.300000);
\draw[draw, fill=black!19] (0.900000,0.300000) rectangle (1.200000, 0.000000);
\draw[draw, fill=black!34] (1.200000,2.400000) rectangle (1.500000, 2.100000);
\draw[draw, fill=black!25] (1.200000,2.100000) rectangle (1.500000, 1.800000);
\draw[draw, fill=black!25] (1.200000,1.800000) rectangle (1.500000, 1.500000);
\draw[draw, fill=black!14] (1.200000,1.500000) rectangle (1.500000, 1.200000);
\draw[draw, fill=black!14] (1.200000,1.200000) rectangle (1.500000, 0.900000);
\draw[draw, fill=black!11] (1.200000,0.900000) rectangle (1.500000, 0.600000);
\draw[draw, fill=black!19] (1.200000,0.600000) rectangle (1.500000, 0.300000);
\draw[draw, fill=black!19] (1.200000,0.300000) rectangle (1.500000, 0.000000);
\draw[draw, fill=black!19] (1.500000,2.400000) rectangle (1.800000, 2.100000);
\draw[draw, fill=black!19] (1.500000,2.100000) rectangle (1.800000, 1.800000);
\draw[draw, fill=black!19] (1.500000,1.800000) rectangle (1.800000, 1.500000);
\draw[draw, fill=black!11] (1.500000,1.500000) rectangle (1.800000, 1.200000);
\draw[draw, fill=black!11] (1.500000,1.200000) rectangle (1.800000, 0.900000);
\draw[draw, fill=black!6] (1.500000,0.900000) rectangle (1.800000, 0.600000);
\draw[draw, fill=black!11] (1.500000,0.600000) rectangle (1.800000, 0.300000);
\draw[draw, fill=black!11] (1.500000,0.300000) rectangle (1.800000, 0.000000);
\draw[draw, fill=black!25] (1.800000,2.400000) rectangle (2.100000, 2.100000);
\draw[draw, fill=black!34] (1.800000,2.100000) rectangle (2.100000, 1.800000);
\draw[draw, fill=black!25] (1.800000,1.800000) rectangle (2.100000, 1.500000);
\draw[draw, fill=black!19] (1.800000,1.500000) rectangle (2.100000, 1.200000);
\draw[draw, fill=black!19] (1.800000,1.200000) rectangle (2.100000, 0.900000);
\draw[draw, fill=black!11] (1.800000,0.900000) rectangle (2.100000, 0.600000);
\draw[draw, fill=black!14] (1.800000,0.600000) rectangle (2.100000, 0.300000);
\draw[draw, fill=black!14] (1.800000,0.300000) rectangle (2.100000, 0.000000);
\draw[draw, fill=black!25] (2.100000,2.400000) rectangle (2.400000, 2.100000);
\draw[draw, fill=black!34] (2.100000,2.100000) rectangle (2.400000, 1.800000);
\draw[draw, fill=black!25] (2.100000,1.800000) rectangle (2.400000, 1.500000);
\draw[draw, fill=black!19] (2.100000,1.500000) rectangle (2.400000, 1.200000);
\draw[draw, fill=black!19] (2.100000,1.200000) rectangle (2.400000, 0.900000);
\draw[draw, fill=black!11] (2.100000,0.900000) rectangle (2.400000, 0.600000);
\draw[draw, fill=black!14] (2.100000,0.600000) rectangle (2.400000, 0.300000);
\draw[draw, fill=black!14] (2.100000,0.300000) rectangle (2.400000, 0.000000);
  \end{tikzpicture}
  \caption{Examples of edge probability matrix (Left: KPGM, Right:
    MAGM).  Darker cells imply a higher probability of observing an
    edge.  On the left, one can see the fractal-like Kronecker
    structure; dividing the matrix into four equal
    parts yields four sub-matrices which up to a scaling look exactly
    the same.
  }

  \label{fig:ep_matrix}
\end{figure}
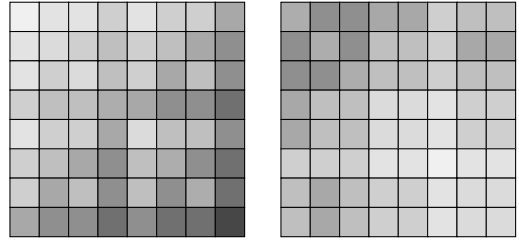

Note that one can generalize the model in two ways
\citep{LesChaKleFaletal10}. First, one can use larger initiator
matrices. Second, different initiator matrices $\Theta^{(1)},
\Theta^{(2)}, \ldots, \Theta^{(d)}$ can be used at each level. In this
case, the edge probability matrix becomes
\begin{align}
  \label{eq:kpgm_p2}
  P = \Theta^{(1)} \otimes \Theta^{(2)} \otimes \cdots \otimes
  \Theta^{(d)}.
\end{align}
In this paper we will work with \eqref{eq:kpgm_p2} because it is closer
in spirit to the MAGM which we will introduce later. For notational
simplicity, we will denote
\begin{align}
  \label{eq:thetat-def}
  \Thetat := \cbr{ \Theta^{(1)}, \Theta^{(2)}, \ldots, \Theta^{(d)}}.
\end{align}

\subsection{Sampling}
\label{sec:Sampling}

The straightforward way to sample a graph from a KPGM is to
\emph{individually} sample each entry $A_{ij}$ of the adjacency matrix
$A$ independently.  This is because, given the parameter matrix
$\Theta$, the event of observing an edge between nodes $i$ and $j$
is independent of observing an edge between nodes
$i'$ and $j'$ for $(i, j) \neq (i', j')$. Thus one can view the
adjacency matrix $A$ as a $n \times n$ array of 
independent Bernoulli random
variables, with $\PP\rbr{ A_{ij} = 1 \mid \Theta } = P_{ij}$.  Such an
algorithm requires $O(n^2)$ effort.

However, the Kronecker structure in the edge probability matrix $P$ can be
exploited to sample a graph in expected $O(\log_{2}(n) |E|)$ time.
The idea of Algorithm~\ref{alg:kpgmsample} suggested by
\citet{LesChaKleFaletal10} is as follows:

The algorithm first determines the number of edges in the graph.  Since
the number of edges $|E| = \sum_{i,j} A_{ij}$ is the sum of $n^2$
independent Bernoulli random variables, it approximately follows the
normal distribution for large $n$.  Thus, one can sample the number of
edges according to this normal distribution.

Next, the algorithm samples each individual edge according to the
following recursive scheme.  Let $S$ denote the set of candidate nodes
for the source and $T$ the candidate target nodes. Initially both $S$
and $T$ are $\{1,2,\ldots,n\}$.  Using the matrix $\Theta$, the
proportion of expected number of edges in each quadrisection
(north-west, north-east, south-west and south-east) of the adjacency
matrix can be computed via
\begin{align}
  \label{eq:quadrant-exp-edges}
  \sum_{i=an/2+1}^{(a+1)n/2} \sum_{j=bn/2+1}^{(b+1)n/2} P_{ij} \propto
  \theta^{(1)}_{ab}, \;\; 0 \leq a,b \leq 1.
\end{align}
Then one can sample a pair of integers $(a,b), 0 \leq a,b \leq 1$, with
the probability of $(a,b)$ proportional to $\theta_{ab}$, to reduce $S$
and $T$ to $\{ an/2+1, an/2+2, \ldots, (a+1)n/2 \}$ and $\{ bn/2+1,
bn/2+2, \ldots, (b+1)n/2 \}$, respectively.  Due to the Kronecker
structure of the edge probability matrix $P$, repeating this
quadrisection procedure $d$ times reduces both $S$ and $T$ to single
nodes $S=\{i\}$ and $T=\{j\}$, which are now connected by an edge. There
is a small non-zero probability that the same edge is sampled multiple
times. In this case the generated edge is rejected and a new edge is
sampled (see pseudo-code in Algorithm~\ref{alg:kpgmsample}).

\begin{algorithm}
  \caption{Sampling Algorithm of Stochastic Kronecker Graphs}\label{alg:kpgmsample}
  \begin{algorithmic}[1]
    \Procedure {KPGMSample}{$\Thetat$}
    \State $E \gets \emptyset$
    \State $m \gets \prod_{k=1}^d(\theta^{(k)}_{00} +\theta^{(k)}_{01}
    +\theta^{(k)}_{10}+\theta^{(k)}_{11})$
    \State $v \gets \prod_{k=1}^d(
    (\theta^{(k)}_{00})^2 +(\theta^{(k)}_{01})^2
    +(\theta^{(k)}_{10})^2+(\theta^{(k)}_{11})^2)$
    \State Generate $X \sim \mathcal{N}(
    m, m-v)$
    \For{$x=1$ to $X$}
    \State $S_{start}, T_{start} \gets 1$
    \State $S_{end}, T_{end} \gets n$
    \For{$k \gets 1$ to $d$}
    \State Sample $(a,b) \propto \theta_{ab}^{(k)}$
    \State $S_{start} \gets S_{start} + a \left(\frac n {2^k}\right)$.
    \State $T_{start} \gets T_{start} + b \left(\frac n {2^k} \right)$.
    \State $S_{end} \gets S_{end} - (1-a) \left(\frac n {2^k}\right)$.
    \State $T_{end} \gets T_{end} - (1-b) \left(\frac n {2^k}\right)$.
    \EndFor
    \State {\# We have $S_{start} = S_{end}$,
    $T_{start} = T_{end}$}
    \State $E \gets E \cup \{ (S_{start}, T_{start}) \}$
    \EndFor
    \State \textbf{return} $E$
    \EndProcedure
  \end{algorithmic}
\end{algorithm}

\section{Multiplicative Attribute Graph Model}
\label{sec:mag}

An alternate way to view KPGM is as follows: Associate the $i$-th node
with a bit-vector $b(i)$ of length $\log_{2}(n)$ such that $b_{k}(i)$ is
the $k$-th digit of integer $(i-1)$ in its binary representation. Then
one can verify that the $(i,j)$-th entry of the edge probability matrix
$P$ in \eqref{eq:kpgm_p2} can be written as
\begin{align}
  \label{eq:kpgm_prob}
  P_{ij} = \prod_{k=1}^d \theta^{(k)}_{b_k(i) \; b_k(j)}.
\end{align}
Under this interpretation, one may consider $b_k(i) = 1$ (resp.\ $b_k(i)
= 0$) as denoting the presence (resp.\ absence) of the $k$-th attribute
in node $i$. The factor $\theta_{b_k(i) \; b_k(j)}^{(k)}$ denotes the
probability of an edge between nodes $i$ and $j$ based on the value of
their $k$-th attribute. The attributes are assumed independent, and
therefore the overall probability of an edge between $i$ and $j$ is just
the product of $\theta_{b_k(i) \; b_k(j)}^{(k)}$'s.

The Multiplicative Attribute Graph Model (MAGM) of \citet{KimLes10} is
also obtained by associating a bit-vector $f(i)$ with a node
$i$. However, $f(i)$ need not be the binary representation of $(i-1)$ as
was the case in the KPGM. In fact, $f(i)$ need not even be of length
$\log_{2}(n)$. We simply assume that $f_{k}(i)$ is a Bernoulli random
variable with $\PP\rbr{f_{k}(i)=1} = \mu^{(k)}$. In addition to
$\Thetat$ defined in \eqref{eq:thetat-def}, the model now has additional
parameters $\mut := \cbr{\mu^{(1)}, \mu^{(2)}, \ldots, \mu^{(d)}}$, and
the $(i,j)$-th entry of the edge probability matrix $Q$ is written as
\begin{align}
  \label{eq:mag_prob}
  Q_{ij} = \prod_{k=1}^d \theta^{(k)}_{f_k(i) \; f_k(j)}.
\end{align}

\section{Quilting Algorithm}
\label{sec:algorithm}

A close examination of \eqref{eq:kpgm_prob} and \eqref{eq:mag_prob}
reveals that KPGM and MAGM are very related. The only difference is that
in the case of the KPGM the $i$-th node is mapped to the bit vector
corresponding to $(i-1)$ while in the case of MAGM it is mapped to an
integer $\lambda_i$ (not necessarily $(i-1)$) whose bit vector
representation is $f(i)$. We will call $\lambda_i$ the \emph{attribute
  configuration} of node $i$ in the sequel.

For ease of theoretical analysis and in order to convey our main ideas
we will initially assume that $d = \log_{2}(n)$, that is, we assume that
$f(i)$ is of length $\log_{2}(n)$ or equivalently $0 \leq \lambda_i < n$
(this assumption will be relaxed in Section \ref{sec:case_logn_neq_2d}).
Under the above assumption, every entry of $Q$ has a corresponding
counterpart in $P$ because
\begin{align}
  \label{eq:kpgm_mag_con}
  Q_{ij} = P_{\lambda_{i} \; \lambda_{j}}.
\end{align}
The key difficulty in sampling graphs from MAGM arises because the
attribute configuration associated with different nodes need not be
unique. To sidestep this issue we partition the nodes into $B$ sets such
that no two nodes in a set share the same attribute configuration.


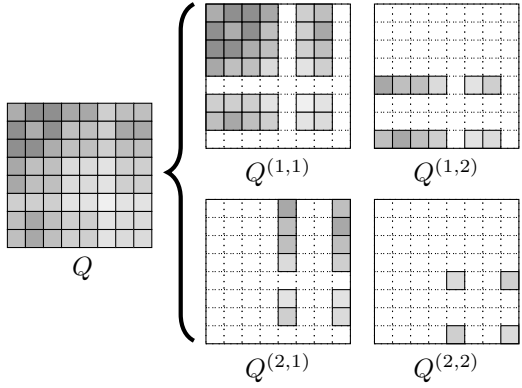
\begin{figure}[htbp]
  \centering
  \begin{tikzpicture}[scale=0.8]
    \begin{scope}[yshift = 0cm, xshift = 0cm]
      \draw[draw, fill=black!32] (0.000000,2.400000) rectangle (0.300000, 2.100000);
\draw[draw, fill=black!44] (0.000000,2.100000) rectangle (0.300000, 1.800000);
\draw[draw, fill=black!44] (0.000000,1.800000) rectangle (0.300000, 1.500000);
\draw[draw, fill=black!34] (0.000000,1.500000) rectangle (0.300000, 1.200000);
\draw[draw, fill=black!34] (0.000000,1.200000) rectangle (0.300000, 0.900000);
\draw[draw, fill=black!19] (0.000000,0.900000) rectangle (0.300000, 0.600000);
\draw[draw, fill=black!25] (0.000000,0.600000) rectangle (0.300000, 0.300000);
\draw[draw, fill=black!25] (0.000000,0.300000) rectangle (0.300000, 0.000000);
\draw[draw, fill=black!44] (0.300000,2.400000) rectangle (0.600000, 2.100000);
\draw[draw, fill=black!32] (0.300000,2.100000) rectangle (0.600000, 1.800000);
\draw[draw, fill=black!44] (0.300000,1.800000) rectangle (0.600000, 1.500000);
\draw[draw, fill=black!25] (0.300000,1.500000) rectangle (0.600000, 1.200000);
\draw[draw, fill=black!25] (0.300000,1.200000) rectangle (0.600000, 0.900000);
\draw[draw, fill=black!19] (0.300000,0.900000) rectangle (0.600000, 0.600000);
\draw[draw, fill=black!34] (0.300000,0.600000) rectangle (0.600000, 0.300000);
\draw[draw, fill=black!34] (0.300000,0.300000) rectangle (0.600000, 0.000000);
\draw[draw, fill=black!44] (0.600000,2.400000) rectangle (0.900000, 2.100000);
\draw[draw, fill=black!44] (0.600000,2.100000) rectangle (0.900000, 1.800000);
\draw[draw, fill=black!32] (0.600000,1.800000) rectangle (0.900000, 1.500000);
\draw[draw, fill=black!25] (0.600000,1.500000) rectangle (0.900000, 1.200000);
\draw[draw, fill=black!25] (0.600000,1.200000) rectangle (0.900000, 0.900000);
\draw[draw, fill=black!19] (0.600000,0.900000) rectangle (0.900000, 0.600000);
\draw[draw, fill=black!25] (0.600000,0.600000) rectangle (0.900000, 0.300000);
\draw[draw, fill=black!25] (0.600000,0.300000) rectangle (0.900000, 0.000000);
\draw[draw, fill=black!34] (0.900000,2.400000) rectangle (1.200000, 2.100000);
\draw[draw, fill=black!25] (0.900000,2.100000) rectangle (1.200000, 1.800000);
\draw[draw, fill=black!25] (0.900000,1.800000) rectangle (1.200000, 1.500000);
\draw[draw, fill=black!14] (0.900000,1.500000) rectangle (1.200000, 1.200000);
\draw[draw, fill=black!14] (0.900000,1.200000) rectangle (1.200000, 0.900000);
\draw[draw, fill=black!11] (0.900000,0.900000) rectangle (1.200000, 0.600000);
\draw[draw, fill=black!19] (0.900000,0.600000) rectangle (1.200000, 0.300000);
\draw[draw, fill=black!19] (0.900000,0.300000) rectangle (1.200000, 0.000000);
\draw[draw, fill=black!34] (1.200000,2.400000) rectangle (1.500000, 2.100000);
\draw[draw, fill=black!25] (1.200000,2.100000) rectangle (1.500000, 1.800000);
\draw[draw, fill=black!25] (1.200000,1.800000) rectangle (1.500000, 1.500000);
\draw[draw, fill=black!14] (1.200000,1.500000) rectangle (1.500000, 1.200000);
\draw[draw, fill=black!14] (1.200000,1.200000) rectangle (1.500000, 0.900000);
\draw[draw, fill=black!11] (1.200000,0.900000) rectangle (1.500000, 0.600000);
\draw[draw, fill=black!19] (1.200000,0.600000) rectangle (1.500000, 0.300000);
\draw[draw, fill=black!19] (1.200000,0.300000) rectangle (1.500000, 0.000000);
\draw[draw, fill=black!19] (1.500000,2.400000) rectangle (1.800000, 2.100000);
\draw[draw, fill=black!19] (1.500000,2.100000) rectangle (1.800000, 1.800000);
\draw[draw, fill=black!19] (1.500000,1.800000) rectangle (1.800000, 1.500000);
\draw[draw, fill=black!11] (1.500000,1.500000) rectangle (1.800000, 1.200000);
\draw[draw, fill=black!11] (1.500000,1.200000) rectangle (1.800000, 0.900000);
\draw[draw, fill=black!6] (1.500000,0.900000) rectangle (1.800000, 0.600000);
\draw[draw, fill=black!11] (1.500000,0.600000) rectangle (1.800000, 0.300000);
\draw[draw, fill=black!11] (1.500000,0.300000) rectangle (1.800000, 0.000000);
\draw[draw, fill=black!25] (1.800000,2.400000) rectangle (2.100000, 2.100000);
\draw[draw, fill=black!34] (1.800000,2.100000) rectangle (2.100000, 1.800000);
\draw[draw, fill=black!25] (1.800000,1.800000) rectangle (2.100000, 1.500000);
\draw[draw, fill=black!19] (1.800000,1.500000) rectangle (2.100000, 1.200000);
\draw[draw, fill=black!19] (1.800000,1.200000) rectangle (2.100000, 0.900000);
\draw[draw, fill=black!11] (1.800000,0.900000) rectangle (2.100000, 0.600000);
\draw[draw, fill=black!14] (1.800000,0.600000) rectangle (2.100000, 0.300000);
\draw[draw, fill=black!14] (1.800000,0.300000) rectangle (2.100000, 0.000000);
\draw[draw, fill=black!25] (2.100000,2.400000) rectangle (2.400000, 2.100000);
\draw[draw, fill=black!34] (2.100000,2.100000) rectangle (2.400000, 1.800000);
\draw[draw, fill=black!25] (2.100000,1.800000) rectangle (2.400000, 1.500000);
\draw[draw, fill=black!19] (2.100000,1.500000) rectangle (2.400000, 1.200000);
\draw[draw, fill=black!19] (2.100000,1.200000) rectangle (2.400000, 0.900000);
\draw[draw, fill=black!11] (2.100000,0.900000) rectangle (2.400000, 0.600000);
\draw[draw, fill=black!14] (2.100000,0.600000) rectangle (2.400000, 0.300000);
\draw[draw, fill=black!14] (2.100000,0.300000) rectangle (2.400000, 0.000000);
      \draw[below] (1.25, 0) node{{$Q$}};
    \end{scope}
    \draw [decoration={brace, amplitude=10pt}, decorate,
    ultra thick, black] (3.1,-1.55) -- (3.1,4.05);
    \begin{scope}[yshift = 1.65cm, xshift = 3.3cm]
      \draw[draw, fill=black!32] (0.000000,2.400000) rectangle (0.300000, 2.100000);
\draw[draw, fill=black!44] (0.000000,2.100000) rectangle (0.300000, 1.800000);
\draw[draw, fill=black!44] (0.000000,1.800000) rectangle (0.300000, 1.500000);
\draw[draw, fill=black!34] (0.000000,1.500000) rectangle (0.300000, 1.200000);
\draw[dashed, style=dotted] (0.000000,1.200000) rectangle (0.300000, 0.900000);
\draw[draw, fill=black!19] (0.000000,0.900000) rectangle (0.300000, 0.600000);
\draw[draw, fill=black!25] (0.000000,0.600000) rectangle (0.300000, 0.300000);
\draw[dashed, style=dotted] (0.000000,0.300000) rectangle (0.300000, 0.000000);
\draw[draw, fill=black!44] (0.300000,2.400000) rectangle (0.600000, 2.100000);
\draw[draw, fill=black!32] (0.300000,2.100000) rectangle (0.600000, 1.800000);
\draw[draw, fill=black!44] (0.300000,1.800000) rectangle (0.600000, 1.500000);
\draw[draw, fill=black!25] (0.300000,1.500000) rectangle (0.600000, 1.200000);
\draw[dashed, style=dotted] (0.300000,1.200000) rectangle (0.600000, 0.900000);
\draw[draw, fill=black!19] (0.300000,0.900000) rectangle (0.600000, 0.600000);
\draw[draw, fill=black!34] (0.300000,0.600000) rectangle (0.600000, 0.300000);
\draw[dashed, style=dotted] (0.300000,0.300000) rectangle (0.600000, 0.000000);
\draw[draw, fill=black!44] (0.600000,2.400000) rectangle (0.900000, 2.100000);
\draw[draw, fill=black!44] (0.600000,2.100000) rectangle (0.900000, 1.800000);
\draw[draw, fill=black!32] (0.600000,1.800000) rectangle (0.900000, 1.500000);
\draw[draw, fill=black!25] (0.600000,1.500000) rectangle (0.900000, 1.200000);
\draw[dashed, style=dotted] (0.600000,1.200000) rectangle (0.900000, 0.900000);
\draw[draw, fill=black!19] (0.600000,0.900000) rectangle (0.900000, 0.600000);
\draw[draw, fill=black!25] (0.600000,0.600000) rectangle (0.900000, 0.300000);
\draw[dashed, style=dotted] (0.600000,0.300000) rectangle (0.900000, 0.000000);
\draw[draw, fill=black!34] (0.900000,2.400000) rectangle (1.200000, 2.100000);
\draw[draw, fill=black!25] (0.900000,2.100000) rectangle (1.200000, 1.800000);
\draw[draw, fill=black!25] (0.900000,1.800000) rectangle (1.200000, 1.500000);
\draw[draw, fill=black!14] (0.900000,1.500000) rectangle (1.200000, 1.200000);
\draw[dashed, style=dotted] (0.900000,1.200000) rectangle (1.200000, 0.900000);
\draw[draw, fill=black!11] (0.900000,0.900000) rectangle (1.200000, 0.600000);
\draw[draw, fill=black!19] (0.900000,0.600000) rectangle (1.200000, 0.300000);
\draw[dashed, style=dotted] (0.900000,0.300000) rectangle (1.200000, 0.000000);
\draw[dashed, style=dotted] (1.200000,2.400000) rectangle (1.500000, 2.100000);
\draw[dashed, style=dotted] (1.200000,2.100000) rectangle (1.500000, 1.800000);
\draw[dashed, style=dotted] (1.200000,1.800000) rectangle (1.500000, 1.500000);
\draw[dashed, style=dotted] (1.200000,1.500000) rectangle (1.500000, 1.200000);
\draw[dashed, style=dotted] (1.200000,1.200000) rectangle (1.500000, 0.900000);
\draw[dashed, style=dotted] (1.200000,0.900000) rectangle (1.500000, 0.600000);
\draw[dashed, style=dotted] (1.200000,0.600000) rectangle (1.500000, 0.300000);
\draw[dashed, style=dotted] (1.200000,0.300000) rectangle (1.500000, 0.000000);
\draw[draw, fill=black!19] (1.500000,2.400000) rectangle (1.800000, 2.100000);
\draw[draw, fill=black!19] (1.500000,2.100000) rectangle (1.800000, 1.800000);
\draw[draw, fill=black!19] (1.500000,1.800000) rectangle (1.800000, 1.500000);
\draw[draw, fill=black!11] (1.500000,1.500000) rectangle (1.800000, 1.200000);
\draw[dashed, style=dotted] (1.500000,1.200000) rectangle (1.800000, 0.900000);
\draw[draw, fill=black!6] (1.500000,0.900000) rectangle (1.800000, 0.600000);
\draw[draw, fill=black!11] (1.500000,0.600000) rectangle (1.800000, 0.300000);
\draw[dashed, style=dotted] (1.500000,0.300000) rectangle (1.800000, 0.000000);
\draw[draw, fill=black!25] (1.800000,2.400000) rectangle (2.100000, 2.100000);
\draw[draw, fill=black!34] (1.800000,2.100000) rectangle (2.100000, 1.800000);
\draw[draw, fill=black!25] (1.800000,1.800000) rectangle (2.100000, 1.500000);
\draw[draw, fill=black!19] (1.800000,1.500000) rectangle (2.100000, 1.200000);
\draw[dashed, style=dotted] (1.800000,1.200000) rectangle (2.100000, 0.900000);
\draw[draw, fill=black!11] (1.800000,0.900000) rectangle (2.100000, 0.600000);
\draw[draw, fill=black!14] (1.800000,0.600000) rectangle (2.100000, 0.300000);
\draw[dashed, style=dotted] (1.800000,0.300000) rectangle (2.100000, 0.000000);
\draw[dashed, style=dotted] (2.100000,2.400000) rectangle (2.400000, 2.100000);
\draw[dashed, style=dotted] (2.100000,2.100000) rectangle (2.400000, 1.800000);
\draw[dashed, style=dotted] (2.100000,1.800000) rectangle (2.400000, 1.500000);
\draw[dashed, style=dotted] (2.100000,1.500000) rectangle (2.400000, 1.200000);
\draw[dashed, style=dotted] (2.100000,1.200000) rectangle (2.400000, 0.900000);
\draw[dashed, style=dotted] (2.100000,0.900000) rectangle (2.400000, 0.600000);
\draw[dashed, style=dotted] (2.100000,0.600000) rectangle (2.400000, 0.300000);
\draw[dashed, style=dotted] (2.100000,0.300000) rectangle (2.400000, 0.000000);
\draw[draw] (0.000000,0.000000) rectangle (2.400000, 2.400000);
\draw[below] (1.200000,0) node {$Q^{(1,1)}$};
    \end{scope}
    \begin{scope}[yshift = 1.65cm, xshift = 6.1cm]
       \draw[dashed, style=dotted] (0.000000,2.400000) rectangle (0.300000, 2.100000);
\draw[dashed, style=dotted] (0.000000,2.100000) rectangle (0.300000, 1.800000);
\draw[dashed, style=dotted] (0.000000,1.800000) rectangle (0.300000, 1.500000);
\draw[dashed, style=dotted] (0.000000,1.500000) rectangle (0.300000, 1.200000);
\draw[draw, fill=black!34] (0.000000,1.200000) rectangle (0.300000, 0.900000);
\draw[dashed, style=dotted] (0.000000,0.900000) rectangle (0.300000, 0.600000);
\draw[dashed, style=dotted] (0.000000,0.600000) rectangle (0.300000, 0.300000);
\draw[draw, fill=black!25] (0.000000,0.300000) rectangle (0.300000, 0.000000);
\draw[dashed, style=dotted] (0.300000,2.400000) rectangle (0.600000, 2.100000);
\draw[dashed, style=dotted] (0.300000,2.100000) rectangle (0.600000, 1.800000);
\draw[dashed, style=dotted] (0.300000,1.800000) rectangle (0.600000, 1.500000);
\draw[dashed, style=dotted] (0.300000,1.500000) rectangle (0.600000, 1.200000);
\draw[draw, fill=black!25] (0.300000,1.200000) rectangle (0.600000, 0.900000);
\draw[dashed, style=dotted] (0.300000,0.900000) rectangle (0.600000, 0.600000);
\draw[dashed, style=dotted] (0.300000,0.600000) rectangle (0.600000, 0.300000);
\draw[draw, fill=black!34] (0.300000,0.300000) rectangle (0.600000, 0.000000);
\draw[dashed, style=dotted] (0.600000,2.400000) rectangle (0.900000, 2.100000);
\draw[dashed, style=dotted] (0.600000,2.100000) rectangle (0.900000, 1.800000);
\draw[dashed, style=dotted] (0.600000,1.800000) rectangle (0.900000, 1.500000);
\draw[dashed, style=dotted] (0.600000,1.500000) rectangle (0.900000, 1.200000);
\draw[draw, fill=black!25] (0.600000,1.200000) rectangle (0.900000, 0.900000);
\draw[dashed, style=dotted] (0.600000,0.900000) rectangle (0.900000, 0.600000);
\draw[dashed, style=dotted] (0.600000,0.600000) rectangle (0.900000, 0.300000);
\draw[draw, fill=black!25] (0.600000,0.300000) rectangle (0.900000, 0.000000);
\draw[dashed, style=dotted] (0.900000,2.400000) rectangle (1.200000, 2.100000);
\draw[dashed, style=dotted] (0.900000,2.100000) rectangle (1.200000, 1.800000);
\draw[dashed, style=dotted] (0.900000,1.800000) rectangle (1.200000, 1.500000);
\draw[dashed, style=dotted] (0.900000,1.500000) rectangle (1.200000, 1.200000);
\draw[draw, fill=black!14] (0.900000,1.200000) rectangle (1.200000, 0.900000);
\draw[dashed, style=dotted] (0.900000,0.900000) rectangle (1.200000, 0.600000);
\draw[dashed, style=dotted] (0.900000,0.600000) rectangle (1.200000, 0.300000);
\draw[draw, fill=black!19] (0.900000,0.300000) rectangle (1.200000, 0.000000);
\draw[dashed, style=dotted] (1.200000,2.400000) rectangle (1.500000, 2.100000);
\draw[dashed, style=dotted] (1.200000,2.100000) rectangle (1.500000, 1.800000);
\draw[dashed, style=dotted] (1.200000,1.800000) rectangle (1.500000, 1.500000);
\draw[dashed, style=dotted] (1.200000,1.500000) rectangle (1.500000, 1.200000);
\draw[dashed, style=dotted] (1.200000,1.200000) rectangle (1.500000, 0.900000);
\draw[dashed, style=dotted] (1.200000,0.900000) rectangle (1.500000, 0.600000);
\draw[dashed, style=dotted] (1.200000,0.600000) rectangle (1.500000, 0.300000);
\draw[dashed, style=dotted] (1.200000,0.300000) rectangle (1.500000, 0.000000);
\draw[dashed, style=dotted] (1.500000,2.400000) rectangle (1.800000, 2.100000);
\draw[dashed, style=dotted] (1.500000,2.100000) rectangle (1.800000, 1.800000);
\draw[dashed, style=dotted] (1.500000,1.800000) rectangle (1.800000, 1.500000);
\draw[dashed, style=dotted] (1.500000,1.500000) rectangle (1.800000, 1.200000);
\draw[draw, fill=black!11] (1.500000,1.200000) rectangle (1.800000, 0.900000);
\draw[dashed, style=dotted] (1.500000,0.900000) rectangle (1.800000, 0.600000);
\draw[dashed, style=dotted] (1.500000,0.600000) rectangle (1.800000, 0.300000);
\draw[draw, fill=black!11] (1.500000,0.300000) rectangle (1.800000, 0.000000);
\draw[dashed, style=dotted] (1.800000,2.400000) rectangle (2.100000, 2.100000);
\draw[dashed, style=dotted] (1.800000,2.100000) rectangle (2.100000, 1.800000);
\draw[dashed, style=dotted] (1.800000,1.800000) rectangle (2.100000, 1.500000);
\draw[dashed, style=dotted] (1.800000,1.500000) rectangle (2.100000, 1.200000);
\draw[draw, fill=black!19] (1.800000,1.200000) rectangle (2.100000, 0.900000);
\draw[dashed, style=dotted] (1.800000,0.900000) rectangle (2.100000, 0.600000);
\draw[dashed, style=dotted] (1.800000,0.600000) rectangle (2.100000, 0.300000);
\draw[draw, fill=black!14] (1.800000,0.300000) rectangle (2.100000, 0.000000);
\draw[dashed, style=dotted] (2.100000,2.400000) rectangle (2.400000, 2.100000);
\draw[dashed, style=dotted] (2.100000,2.100000) rectangle (2.400000, 1.800000);
\draw[dashed, style=dotted] (2.100000,1.800000) rectangle (2.400000, 1.500000);
\draw[dashed, style=dotted] (2.100000,1.500000) rectangle (2.400000, 1.200000);
\draw[dashed, style=dotted] (2.100000,1.200000) rectangle (2.400000, 0.900000);
\draw[dashed, style=dotted] (2.100000,0.900000) rectangle (2.400000, 0.600000);
\draw[dashed, style=dotted] (2.100000,0.600000) rectangle (2.400000, 0.300000);
\draw[dashed, style=dotted] (2.100000,0.300000) rectangle (2.400000, 0.000000);
\draw[draw] (0.000000,0.000000) rectangle (2.400000, 2.400000);
\draw[below] (1.200000,0) node {$Q^{(1,2)}$};
    \end{scope}
    \begin{scope}[yshift = -1.6cm, xshift = 3.3cm]
      \draw[dashed, style=dotted] (0.000000,2.400000) rectangle (0.300000, 2.100000);
\draw[dashed, style=dotted] (0.000000,2.100000) rectangle (0.300000, 1.800000);
\draw[dashed, style=dotted] (0.000000,1.800000) rectangle (0.300000, 1.500000);
\draw[dashed, style=dotted] (0.000000,1.500000) rectangle (0.300000, 1.200000);
\draw[dashed, style=dotted] (0.000000,1.200000) rectangle (0.300000, 0.900000);
\draw[dashed, style=dotted] (0.000000,0.900000) rectangle (0.300000, 0.600000);
\draw[dashed, style=dotted] (0.000000,0.600000) rectangle (0.300000, 0.300000);
\draw[dashed, style=dotted] (0.000000,0.300000) rectangle (0.300000, 0.000000);
\draw[dashed, style=dotted] (0.300000,2.400000) rectangle (0.600000, 2.100000);
\draw[dashed, style=dotted] (0.300000,2.100000) rectangle (0.600000, 1.800000);
\draw[dashed, style=dotted] (0.300000,1.800000) rectangle (0.600000, 1.500000);
\draw[dashed, style=dotted] (0.300000,1.500000) rectangle (0.600000, 1.200000);
\draw[dashed, style=dotted] (0.300000,1.200000) rectangle (0.600000, 0.900000);
\draw[dashed, style=dotted] (0.300000,0.900000) rectangle (0.600000, 0.600000);
\draw[dashed, style=dotted] (0.300000,0.600000) rectangle (0.600000, 0.300000);
\draw[dashed, style=dotted] (0.300000,0.300000) rectangle (0.600000, 0.000000);
\draw[dashed, style=dotted] (0.600000,2.400000) rectangle (0.900000, 2.100000);
\draw[dashed, style=dotted] (0.600000,2.100000) rectangle (0.900000, 1.800000);
\draw[dashed, style=dotted] (0.600000,1.800000) rectangle (0.900000, 1.500000);
\draw[dashed, style=dotted] (0.600000,1.500000) rectangle (0.900000, 1.200000);
\draw[dashed, style=dotted] (0.600000,1.200000) rectangle (0.900000, 0.900000);
\draw[dashed, style=dotted] (0.600000,0.900000) rectangle (0.900000, 0.600000);
\draw[dashed, style=dotted] (0.600000,0.600000) rectangle (0.900000, 0.300000);
\draw[dashed, style=dotted] (0.600000,0.300000) rectangle (0.900000, 0.000000);
\draw[dashed, style=dotted] (0.900000,2.400000) rectangle (1.200000, 2.100000);
\draw[dashed, style=dotted] (0.900000,2.100000) rectangle (1.200000, 1.800000);
\draw[dashed, style=dotted] (0.900000,1.800000) rectangle (1.200000, 1.500000);
\draw[dashed, style=dotted] (0.900000,1.500000) rectangle (1.200000, 1.200000);
\draw[dashed, style=dotted] (0.900000,1.200000) rectangle (1.200000, 0.900000);
\draw[dashed, style=dotted] (0.900000,0.900000) rectangle (1.200000, 0.600000);
\draw[dashed, style=dotted] (0.900000,0.600000) rectangle (1.200000, 0.300000);
\draw[dashed, style=dotted] (0.900000,0.300000) rectangle (1.200000, 0.000000);
\draw[draw, fill=black!34] (1.200000,2.400000) rectangle (1.500000, 2.100000);
\draw[draw, fill=black!25] (1.200000,2.100000) rectangle (1.500000, 1.800000);
\draw[draw, fill=black!25] (1.200000,1.800000) rectangle (1.500000, 1.500000);
\draw[draw, fill=black!14] (1.200000,1.500000) rectangle (1.500000, 1.200000);
\draw[dashed, style=dotted] (1.200000,1.200000) rectangle (1.500000, 0.900000);
\draw[draw, fill=black!11] (1.200000,0.900000) rectangle (1.500000, 0.600000);
\draw[draw, fill=black!19] (1.200000,0.600000) rectangle (1.500000, 0.300000);
\draw[dashed, style=dotted] (1.200000,0.300000) rectangle (1.500000, 0.000000);
\draw[dashed, style=dotted] (1.500000,2.400000) rectangle (1.800000, 2.100000);
\draw[dashed, style=dotted] (1.500000,2.100000) rectangle (1.800000, 1.800000);
\draw[dashed, style=dotted] (1.500000,1.800000) rectangle (1.800000, 1.500000);
\draw[dashed, style=dotted] (1.500000,1.500000) rectangle (1.800000, 1.200000);
\draw[dashed, style=dotted] (1.500000,1.200000) rectangle (1.800000, 0.900000);
\draw[dashed, style=dotted] (1.500000,0.900000) rectangle (1.800000, 0.600000);
\draw[dashed, style=dotted] (1.500000,0.600000) rectangle (1.800000, 0.300000);
\draw[dashed, style=dotted] (1.500000,0.300000) rectangle (1.800000, 0.000000);
\draw[dashed, style=dotted] (1.800000,2.400000) rectangle (2.100000, 2.100000);
\draw[dashed, style=dotted] (1.800000,2.100000) rectangle (2.100000, 1.800000);
\draw[dashed, style=dotted] (1.800000,1.800000) rectangle (2.100000, 1.500000);
\draw[dashed, style=dotted] (1.800000,1.500000) rectangle (2.100000, 1.200000);
\draw[dashed, style=dotted] (1.800000,1.200000) rectangle (2.100000, 0.900000);
\draw[dashed, style=dotted] (1.800000,0.900000) rectangle (2.100000, 0.600000);
\draw[dashed, style=dotted] (1.800000,0.600000) rectangle (2.100000, 0.300000);
\draw[dashed, style=dotted] (1.800000,0.300000) rectangle (2.100000, 0.000000);
\draw[draw, fill=black!25] (2.100000,2.400000) rectangle (2.400000, 2.100000);
\draw[draw, fill=black!34] (2.100000,2.100000) rectangle (2.400000, 1.800000);
\draw[draw, fill=black!25] (2.100000,1.800000) rectangle (2.400000, 1.500000);
\draw[draw, fill=black!19] (2.100000,1.500000) rectangle (2.400000, 1.200000);
\draw[dashed, style=dotted] (2.100000,1.200000) rectangle (2.400000, 0.900000);
\draw[draw, fill=black!11] (2.100000,0.900000) rectangle (2.400000, 0.600000);
\draw[draw, fill=black!14] (2.100000,0.600000) rectangle (2.400000, 0.300000);
\draw[dashed, style=dotted] (2.100000,0.300000) rectangle (2.400000, 0.000000);
\draw[draw] (0.000000,0.000000) rectangle (2.400000, 2.400000);
\draw[below] (1.200000,0) node {$Q^{(2,1)}$};
    \end{scope}
    \begin{scope}[yshift = -1.6cm, xshift = 6.1cm]
      \draw[dashed, style=dotted] (0.000000,2.400000) rectangle (0.300000, 2.100000);
\draw[dashed, style=dotted] (0.000000,2.100000) rectangle (0.300000, 1.800000);
\draw[dashed, style=dotted] (0.000000,1.800000) rectangle (0.300000, 1.500000);
\draw[dashed, style=dotted] (0.000000,1.500000) rectangle (0.300000, 1.200000);
\draw[dashed, style=dotted] (0.000000,1.200000) rectangle (0.300000, 0.900000);
\draw[dashed, style=dotted] (0.000000,0.900000) rectangle (0.300000, 0.600000);
\draw[dashed, style=dotted] (0.000000,0.600000) rectangle (0.300000, 0.300000);
\draw[dashed, style=dotted] (0.000000,0.300000) rectangle (0.300000, 0.000000);
\draw[dashed, style=dotted] (0.300000,2.400000) rectangle (0.600000, 2.100000);
\draw[dashed, style=dotted] (0.300000,2.100000) rectangle (0.600000, 1.800000);
\draw[dashed, style=dotted] (0.300000,1.800000) rectangle (0.600000, 1.500000);
\draw[dashed, style=dotted] (0.300000,1.500000) rectangle (0.600000, 1.200000);
\draw[dashed, style=dotted] (0.300000,1.200000) rectangle (0.600000, 0.900000);
\draw[dashed, style=dotted] (0.300000,0.900000) rectangle (0.600000, 0.600000);
\draw[dashed, style=dotted] (0.300000,0.600000) rectangle (0.600000, 0.300000);
\draw[dashed, style=dotted] (0.300000,0.300000) rectangle (0.600000, 0.000000);
\draw[dashed, style=dotted] (0.600000,2.400000) rectangle (0.900000, 2.100000);
\draw[dashed, style=dotted] (0.600000,2.100000) rectangle (0.900000, 1.800000);
\draw[dashed, style=dotted] (0.600000,1.800000) rectangle (0.900000, 1.500000);
\draw[dashed, style=dotted] (0.600000,1.500000) rectangle (0.900000, 1.200000);
\draw[dashed, style=dotted] (0.600000,1.200000) rectangle (0.900000, 0.900000);
\draw[dashed, style=dotted] (0.600000,0.900000) rectangle (0.900000, 0.600000);
\draw[dashed, style=dotted] (0.600000,0.600000) rectangle (0.900000, 0.300000);
\draw[dashed, style=dotted] (0.600000,0.300000) rectangle (0.900000, 0.000000);
\draw[dashed, style=dotted] (0.900000,2.400000) rectangle (1.200000, 2.100000);
\draw[dashed, style=dotted] (0.900000,2.100000) rectangle (1.200000, 1.800000);
\draw[dashed, style=dotted] (0.900000,1.800000) rectangle (1.200000, 1.500000);
\draw[dashed, style=dotted] (0.900000,1.500000) rectangle (1.200000, 1.200000);
\draw[dashed, style=dotted] (0.900000,1.200000) rectangle (1.200000, 0.900000);
\draw[dashed, style=dotted] (0.900000,0.900000) rectangle (1.200000, 0.600000);
\draw[dashed, style=dotted] (0.900000,0.600000) rectangle (1.200000, 0.300000);
\draw[dashed, style=dotted] (0.900000,0.300000) rectangle (1.200000, 0.000000);
\draw[dashed, style=dotted] (1.200000,2.400000) rectangle (1.500000, 2.100000);
\draw[dashed, style=dotted] (1.200000,2.100000) rectangle (1.500000, 1.800000);
\draw[dashed, style=dotted] (1.200000,1.800000) rectangle (1.500000, 1.500000);
\draw[dashed, style=dotted] (1.200000,1.500000) rectangle (1.500000, 1.200000);
\draw[draw, fill=black!14] (1.200000,1.200000) rectangle (1.500000, 0.900000);
\draw[dashed, style=dotted] (1.200000,0.900000) rectangle (1.500000, 0.600000);
\draw[dashed, style=dotted] (1.200000,0.600000) rectangle (1.500000, 0.300000);
\draw[draw, fill=black!19] (1.200000,0.300000) rectangle (1.500000, 0.000000);
\draw[dashed, style=dotted] (1.500000,2.400000) rectangle (1.800000, 2.100000);
\draw[dashed, style=dotted] (1.500000,2.100000) rectangle (1.800000, 1.800000);
\draw[dashed, style=dotted] (1.500000,1.800000) rectangle (1.800000, 1.500000);
\draw[dashed, style=dotted] (1.500000,1.500000) rectangle (1.800000, 1.200000);
\draw[dashed, style=dotted] (1.500000,1.200000) rectangle (1.800000, 0.900000);
\draw[dashed, style=dotted] (1.500000,0.900000) rectangle (1.800000, 0.600000);
\draw[dashed, style=dotted] (1.500000,0.600000) rectangle (1.800000, 0.300000);
\draw[dashed, style=dotted] (1.500000,0.300000) rectangle (1.800000, 0.000000);
\draw[dashed, style=dotted] (1.800000,2.400000) rectangle (2.100000, 2.100000);
\draw[dashed, style=dotted] (1.800000,2.100000) rectangle (2.100000, 1.800000);
\draw[dashed, style=dotted] (1.800000,1.800000) rectangle (2.100000, 1.500000);
\draw[dashed, style=dotted] (1.800000,1.500000) rectangle (2.100000, 1.200000);
\draw[dashed, style=dotted] (1.800000,1.200000) rectangle (2.100000, 0.900000);
\draw[dashed, style=dotted] (1.800000,0.900000) rectangle (2.100000, 0.600000);
\draw[dashed, style=dotted] (1.800000,0.600000) rectangle (2.100000, 0.300000);
\draw[dashed, style=dotted] (1.800000,0.300000) rectangle (2.100000, 0.000000);
\draw[dashed, style=dotted] (2.100000,2.400000) rectangle (2.400000, 2.100000);
\draw[dashed, style=dotted] (2.100000,2.100000) rectangle (2.400000, 1.800000);
\draw[dashed, style=dotted] (2.100000,1.800000) rectangle (2.400000, 1.500000);
\draw[dashed, style=dotted] (2.100000,1.500000) rectangle (2.400000, 1.200000);
\draw[draw, fill=black!19] (2.100000,1.200000) rectangle (2.400000, 0.900000);
\draw[dashed, style=dotted] (2.100000,0.900000) rectangle (2.400000, 0.600000);
\draw[dashed, style=dotted] (2.100000,0.600000) rectangle (2.400000, 0.300000);
\draw[draw, fill=black!14] (2.100000,0.300000) rectangle (2.400000, 0.000000);
\draw[draw] (0.000000,0.000000) rectangle (2.400000, 2.400000);
\draw[below] (1.200000,0) node {$Q^{(2,2)}$};
    \end{scope}

  \end{tikzpicture}
  \caption{Partition the MAGM edge probability matrix $Q$ into $B^2$
    pieces such that no two nodes in a piece share the same attribute
    configuration. }
  \label{fig:quilt}
\end{figure}

While a number of partitioning schemes can be used, we find the
following scheme to be easy to analyze and efficient in practice: For
each $i$ define the set $Z_{i} := \cbr{ j \text{ s.t. } j \leq i \text{
    and } \lambda_i = \lambda_j }$. Clearly $\abr{Z_{i}}$ counts nodes
$j$ whose index is smaller than or equal to $i$ and which share the same
attribute configuration $\lambda_i$. We now define the $c$-th set
$D_{c}$ in our partition as $D_{c} := \cbr{i \text{ s.t. } \abr{Z_{i}} =
  c}$. No two nodes in $D_{c}$ share the same attribute
configuration. Furthermore, one can show that the number of sets $B$ is
minimized by this scheme.
\begin{theorem}[Size Optimality of the Partition]
  \label{thm:optimal_partition}
  Let the partition of $\cbr{1, \ldots, n}$ obtained by the above scheme
  be denoted as $D_1, D_2, \ldots, D_B$. Then, $B$, the number of
  nonempty sets in the partition, is minimized.
\end{theorem}
\begin{proof}
  See Appendix~\ref{sec:TechnicalProofs}. 
\end{proof}

Using the partition $D_1, \ldots, D_B$, we can partition the edge
probability matrix $Q$ into $B^2$ sub-matrices (Figure~\ref{fig:quilt}):
\begin{align*}
  Q^{(k,l)}_{i,j} =
  \begin{cases}
    Q_{i,j} & \text{ if } i \in D_k, j \in D_l, \\
    0 & \text{ otherwise}.
  \end{cases}
\end{align*}

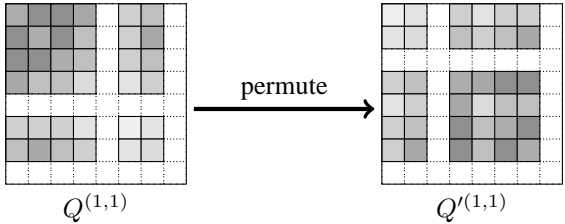
\begin{figure}[htbp]
  \centering
  \begin{tikzpicture}[scale=1.0]
    \begin{scope}[yshift = 0cm, xshift = 0cm]
      \draw[draw, fill=black!32] (0.000000,2.400000) rectangle (0.300000, 2.100000);
\draw[draw, fill=black!44] (0.000000,2.100000) rectangle (0.300000, 1.800000);
\draw[draw, fill=black!44] (0.000000,1.800000) rectangle (0.300000, 1.500000);
\draw[draw, fill=black!34] (0.000000,1.500000) rectangle (0.300000, 1.200000);
\draw[dashed, style=dotted] (0.000000,1.200000) rectangle (0.300000, 0.900000);
\draw[draw, fill=black!19] (0.000000,0.900000) rectangle (0.300000, 0.600000);
\draw[draw, fill=black!25] (0.000000,0.600000) rectangle (0.300000, 0.300000);
\draw[dashed, style=dotted] (0.000000,0.300000) rectangle (0.300000, 0.000000);
\draw[draw, fill=black!44] (0.300000,2.400000) rectangle (0.600000, 2.100000);
\draw[draw, fill=black!32] (0.300000,2.100000) rectangle (0.600000, 1.800000);
\draw[draw, fill=black!44] (0.300000,1.800000) rectangle (0.600000, 1.500000);
\draw[draw, fill=black!25] (0.300000,1.500000) rectangle (0.600000, 1.200000);
\draw[dashed, style=dotted] (0.300000,1.200000) rectangle (0.600000, 0.900000);
\draw[draw, fill=black!19] (0.300000,0.900000) rectangle (0.600000, 0.600000);
\draw[draw, fill=black!34] (0.300000,0.600000) rectangle (0.600000, 0.300000);
\draw[dashed, style=dotted] (0.300000,0.300000) rectangle (0.600000, 0.000000);
\draw[draw, fill=black!44] (0.600000,2.400000) rectangle (0.900000, 2.100000);
\draw[draw, fill=black!44] (0.600000,2.100000) rectangle (0.900000, 1.800000);
\draw[draw, fill=black!32] (0.600000,1.800000) rectangle (0.900000, 1.500000);
\draw[draw, fill=black!25] (0.600000,1.500000) rectangle (0.900000, 1.200000);
\draw[dashed, style=dotted] (0.600000,1.200000) rectangle (0.900000, 0.900000);
\draw[draw, fill=black!19] (0.600000,0.900000) rectangle (0.900000, 0.600000);
\draw[draw, fill=black!25] (0.600000,0.600000) rectangle (0.900000, 0.300000);
\draw[dashed, style=dotted] (0.600000,0.300000) rectangle (0.900000, 0.000000);
\draw[draw, fill=black!34] (0.900000,2.400000) rectangle (1.200000, 2.100000);
\draw[draw, fill=black!25] (0.900000,2.100000) rectangle (1.200000, 1.800000);
\draw[draw, fill=black!25] (0.900000,1.800000) rectangle (1.200000, 1.500000);
\draw[draw, fill=black!14] (0.900000,1.500000) rectangle (1.200000, 1.200000);
\draw[dashed, style=dotted] (0.900000,1.200000) rectangle (1.200000, 0.900000);
\draw[draw, fill=black!11] (0.900000,0.900000) rectangle (1.200000, 0.600000);
\draw[draw, fill=black!19] (0.900000,0.600000) rectangle (1.200000, 0.300000);
\draw[dashed, style=dotted] (0.900000,0.300000) rectangle (1.200000, 0.000000);
\draw[dashed, style=dotted] (1.200000,2.400000) rectangle (1.500000, 2.100000);
\draw[dashed, style=dotted] (1.200000,2.100000) rectangle (1.500000, 1.800000);
\draw[dashed, style=dotted] (1.200000,1.800000) rectangle (1.500000, 1.500000);
\draw[dashed, style=dotted] (1.200000,1.500000) rectangle (1.500000, 1.200000);
\draw[dashed, style=dotted] (1.200000,1.200000) rectangle (1.500000, 0.900000);
\draw[dashed, style=dotted] (1.200000,0.900000) rectangle (1.500000, 0.600000);
\draw[dashed, style=dotted] (1.200000,0.600000) rectangle (1.500000, 0.300000);
\draw[dashed, style=dotted] (1.200000,0.300000) rectangle (1.500000, 0.000000);
\draw[draw, fill=black!19] (1.500000,2.400000) rectangle (1.800000, 2.100000);
\draw[draw, fill=black!19] (1.500000,2.100000) rectangle (1.800000, 1.800000);
\draw[draw, fill=black!19] (1.500000,1.800000) rectangle (1.800000, 1.500000);
\draw[draw, fill=black!11] (1.500000,1.500000) rectangle (1.800000, 1.200000);
\draw[dashed, style=dotted] (1.500000,1.200000) rectangle (1.800000, 0.900000);
\draw[draw, fill=black!6] (1.500000,0.900000) rectangle (1.800000, 0.600000);
\draw[draw, fill=black!11] (1.500000,0.600000) rectangle (1.800000, 0.300000);
\draw[dashed, style=dotted] (1.500000,0.300000) rectangle (1.800000, 0.000000);
\draw[draw, fill=black!25] (1.800000,2.400000) rectangle (2.100000, 2.100000);
\draw[draw, fill=black!34] (1.800000,2.100000) rectangle (2.100000, 1.800000);
\draw[draw, fill=black!25] (1.800000,1.800000) rectangle (2.100000, 1.500000);
\draw[draw, fill=black!19] (1.800000,1.500000) rectangle (2.100000, 1.200000);
\draw[dashed, style=dotted] (1.800000,1.200000) rectangle (2.100000, 0.900000);
\draw[draw, fill=black!11] (1.800000,0.900000) rectangle (2.100000, 0.600000);
\draw[draw, fill=black!14] (1.800000,0.600000) rectangle (2.100000, 0.300000);
\draw[dashed, style=dotted] (1.800000,0.300000) rectangle (2.100000, 0.000000);
\draw[dashed, style=dotted] (2.100000,2.400000) rectangle (2.400000, 2.100000);
\draw[dashed, style=dotted] (2.100000,2.100000) rectangle (2.400000, 1.800000);
\draw[dashed, style=dotted] (2.100000,1.800000) rectangle (2.400000, 1.500000);
\draw[dashed, style=dotted] (2.100000,1.500000) rectangle (2.400000, 1.200000);
\draw[dashed, style=dotted] (2.100000,1.200000) rectangle (2.400000, 0.900000);
\draw[dashed, style=dotted] (2.100000,0.900000) rectangle (2.400000, 0.600000);
\draw[dashed, style=dotted] (2.100000,0.600000) rectangle (2.400000, 0.300000);
\draw[dashed, style=dotted] (2.100000,0.300000) rectangle (2.400000, 0.000000);
\draw[draw] (0.000000,0.000000) rectangle (2.400000, 2.400000);
\draw[below] (1.200000,0) node {$Q^{(1,1)}$};
    \end{scope}
    \begin{scope}[yshift = 0cm, xshift = 5.0cm]
      \draw[draw, fill=black!6] (0.000000,2.400000) rectangle (0.300000, 2.100000);
  \draw[draw, fill=black!11] (0.000000,2.100000) rectangle (0.300000, 1.800000);
  \draw[dashed, style=dotted] (0.000000,1.800000) rectangle (0.300000, 1.500000);
\draw[draw, fill=black!19] (0.000000,1.500000) rectangle (0.300000, 1.200000);
  \draw[draw, fill=black!11] (0.000000,1.200000) rectangle (0.300000, 0.900000);
  \draw[draw, fill=black!19] (0.000000,0.900000) rectangle (0.300000, 0.600000);
  \draw[draw, fill=black!19] (0.000000,0.600000) rectangle (0.300000, 0.300000);
  \draw[dashed, style=dotted] (0.000000,0.300000) rectangle (0.300000, 0.000000);
\draw[draw, fill=black!11] (0.300000,2.400000) rectangle (0.600000, 2.100000);
  \draw[draw, fill=black!14] (0.300000,2.100000) rectangle (0.600000, 1.800000);
  \draw[dashed, style=dotted] (0.300000,1.800000) rectangle (0.600000, 1.500000);
\draw[draw, fill=black!25] (0.300000,1.500000) rectangle (0.600000, 1.200000);
  \draw[draw, fill=black!19] (0.300000,1.200000) rectangle (0.600000, 0.900000);
  \draw[draw, fill=black!25] (0.300000,0.900000) rectangle (0.600000, 0.600000);
  \draw[draw, fill=black!34] (0.300000,0.600000) rectangle (0.600000, 0.300000);
  \draw[dashed, style=dotted] (0.300000,0.300000) rectangle (0.600000, 0.000000);
\draw[dashed, style=dotted] (0.600000,2.400000) rectangle (0.900000, 2.100000);
\draw[dashed, style=dotted] (0.600000,2.100000) rectangle (0.900000, 1.800000);
\draw[dashed, style=dotted] (0.600000,1.800000) rectangle (0.900000, 1.500000);
\draw[dashed, style=dotted] (0.600000,1.500000) rectangle (0.900000, 1.200000);
\draw[dashed, style=dotted] (0.600000,1.200000) rectangle (0.900000, 0.900000);
\draw[dashed, style=dotted] (0.600000,0.900000) rectangle (0.900000, 0.600000);
\draw[dashed, style=dotted] (0.600000,0.600000) rectangle (0.900000, 0.300000);
\draw[dashed, style=dotted] (0.600000,0.300000) rectangle (0.900000, 0.000000);
\draw[draw, fill=black!19] (0.900000,2.400000) rectangle (1.200000, 2.100000);
  \draw[draw, fill=black!25] (0.900000,2.100000) rectangle (1.200000, 1.800000);
  \draw[dashed, style=dotted] (0.900000,1.800000) rectangle (1.200000, 1.500000);
\draw[draw, fill=black!32] (0.900000,1.500000) rectangle (1.200000, 1.200000);
  \draw[draw, fill=black!34] (0.900000,1.200000) rectangle (1.200000, 0.900000);
  \draw[draw, fill=black!44] (0.900000,0.900000) rectangle (1.200000, 0.600000);
  \draw[draw, fill=black!44] (0.900000,0.600000) rectangle (1.200000, 0.300000);
  \draw[dashed, style=dotted] (0.900000,0.300000) rectangle (1.200000, 0.000000);
\draw[draw, fill=black!11] (1.200000,2.400000) rectangle (1.500000, 2.100000);
  \draw[draw, fill=black!19] (1.200000,2.100000) rectangle (1.500000, 1.800000);
  \draw[dashed, style=dotted] (1.200000,1.800000) rectangle (1.500000, 1.500000);
\draw[draw, fill=black!34] (1.200000,1.500000) rectangle (1.500000, 1.200000);
  \draw[draw, fill=black!14] (1.200000,1.200000) rectangle (1.500000, 0.900000);
  \draw[draw, fill=black!25] (1.200000,0.900000) rectangle (1.500000, 0.600000);
  \draw[draw, fill=black!25] (1.200000,0.600000) rectangle (1.500000, 0.300000);
  \draw[dashed, style=dotted] (1.200000,0.300000) rectangle (1.500000, 0.000000);
\draw[draw, fill=black!19] (1.500000,2.400000) rectangle (1.800000, 2.100000);
  \draw[draw, fill=black!25] (1.500000,2.100000) rectangle (1.800000, 1.800000);
  \draw[dashed, style=dotted] (1.500000,1.800000) rectangle (1.800000, 1.500000);
\draw[draw, fill=black!44] (1.500000,1.500000) rectangle (1.800000, 1.200000);
  \draw[draw, fill=black!25] (1.500000,1.200000) rectangle (1.800000, 0.900000);
  \draw[draw, fill=black!32] (1.500000,0.900000) rectangle (1.800000, 0.600000);
  \draw[draw, fill=black!44] (1.500000,0.600000) rectangle (1.800000, 0.300000);
  \draw[dashed, style=dotted] (1.500000,0.300000) rectangle (1.800000, 0.000000);
\draw[draw, fill=black!19] (1.800000,2.400000) rectangle (2.100000, 2.100000);
  \draw[draw, fill=black!34] (1.800000,2.100000) rectangle (2.100000, 1.800000);
  \draw[dashed, style=dotted] (1.800000,1.800000) rectangle (2.100000, 1.500000);
\draw[draw, fill=black!44] (1.800000,1.500000) rectangle (2.100000, 1.200000);
  \draw[draw, fill=black!25] (1.800000,1.200000) rectangle (2.100000, 0.900000);
  \draw[draw, fill=black!44] (1.800000,0.900000) rectangle (2.100000, 0.600000);
  \draw[draw, fill=black!32] (1.800000,0.600000) rectangle (2.100000, 0.300000);
  \draw[dashed, style=dotted] (1.800000,0.300000) rectangle (2.100000, 0.000000);
\draw[dashed, style=dotted] (2.100000,2.400000) rectangle (2.400000, 2.100000);
\draw[dashed, style=dotted] (2.100000,2.100000) rectangle (2.400000, 1.800000);
\draw[dashed, style=dotted] (2.100000,1.800000) rectangle (2.400000, 1.500000);
\draw[dashed, style=dotted] (2.100000,1.500000) rectangle (2.400000, 1.200000);
\draw[dashed, style=dotted] (2.100000,1.200000) rectangle (2.400000, 0.900000);
\draw[dashed, style=dotted] (2.100000,0.900000) rectangle (2.400000, 0.600000);
\draw[dashed, style=dotted] (2.100000,0.600000) rectangle (2.400000, 0.300000);
\draw[dashed, style=dotted] (2.100000,0.300000) rectangle (2.400000, 0.000000);
\draw[draw] (0.000000,0.000000) rectangle (2.400000, 2.400000);
\draw[below] (1.200000,0) node {$Q'^{(1,1)}$};
    \end{scope}
    \draw (2.5, 1.0) edge[->, ultra thick] (4.9, 1.0);
    \draw[above] (3.7, 1.0) node{{permute}};
  \end{tikzpicture}
  \caption{Each piece of the edge probability matrix is permuted to
    become a sub-matrix of the KPGM edge probability matrix. One can
    then apply Algorithm \ref{alg:kpgmsample} to sample graphs from this
    permuted edge probability matrix and retain the sub-graph of
    interest.  }
  \label{fig:quilt2}
\end{figure}

Next, by applying a permutation which maps $\lambda_i$ to $i$ we can
transform each of the $B^2$ sub-matrices of $Q$ into a submatrix of $P$
as illustrated in Figure~\ref{fig:quilt2}. Formally, define
\begin{align*}
  Q'^{(k,l)}_{i,j} =
  \begin{cases}
    Q_{x,y} & \text{ if } x \in D_k, y \in D_l, i=\lambda_x,
    j=\lambda_y \\
    0 & \text{ otherwise}.
  \end{cases}
\end{align*}

Algorithm~\ref{alg:kpgmsample} can be used to sample graphs from this
permuted edge probability matrix with parameters $\Thetat$. We filter
the sampled graph to only retain the sub-graph of interest.  Finally,
the sampled sub-graphs are un-permuted and quilted together to form a
sample from the MAGM (see Figure~\ref{fig:quilt3}). Let $A'^{(k,l)}$
denote the adjacency matrix of the graph sampled from the edge
probability matrix $Q'^{(k,l)}$ via Algorithm
\ref{alg:kpgmsample}. Define
\begin{align}
  A^{(k,l)}_{i,j} =
  \begin{cases}
    A'^{(k,l)}_{x,y} & \text{ if } i \in D_k, j \in D_l,
    x=\lambda_i,
    y=\lambda_j\\
    0 & \text{ otherwise}.
  \end{cases}
  \label{eq:unpermute}
\end{align}
The quilted adjacency matrix $A$ is given by $\sum_{k, l} A^{(k, l)}$.
See Algorithm~\ref{alg:magsample}.

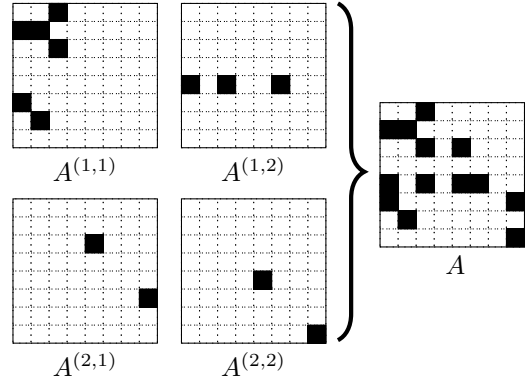
\begin{figure}[h]
  \centering
  \begin{tikzpicture}[scale=0.8]


    \begin{scope}[yshift = 1.65cm, xshift = 3.3cm]
      \draw[dashed, style=dotted] (0.000000,2.400000) rectangle (0.300000, 2.100000);
\draw[dashed, style=dotted] (0.000000,2.100000) rectangle (0.300000, 1.800000);
\draw[dashed, style=dotted] (0.000000,1.800000) rectangle (0.300000, 1.500000);
\draw[dashed, style=dotted] (0.000000,1.500000) rectangle (0.300000, 1.200000);
\draw[dashed, style=dotted] (0.000000,1.200000) rectangle (0.300000, 0.900000);
\draw[dashed, style=dotted] (0.000000,0.900000) rectangle (0.300000, 0.600000);
\draw[dashed, style=dotted] (0.000000,0.600000) rectangle (0.300000, 0.300000);
\draw[dashed, style=dotted] (0.000000,0.300000) rectangle (0.300000, 0.000000);
\draw[dashed, style=dotted] (0.300000,2.400000) rectangle (0.600000, 2.100000);
\draw[dashed, style=dotted] (0.300000,2.100000) rectangle (0.600000, 1.800000);
\draw[dashed, style=dotted] (0.300000,1.800000) rectangle (0.600000, 1.500000);
\draw[dashed, style=dotted] (0.300000,1.500000) rectangle (0.600000, 1.200000);
\draw[dashed, style=dotted] (0.300000,1.200000) rectangle (0.600000, 0.900000);
\draw[dashed, style=dotted] (0.300000,0.900000) rectangle (0.600000, 0.600000);
\draw[dashed, style=dotted] (0.300000,0.600000) rectangle (0.600000, 0.300000);
\draw[dashed, style=dotted] (0.300000,0.300000) rectangle (0.600000, 0.000000);
\draw[dashed, style=dotted] (0.600000,2.400000) rectangle (0.900000, 2.100000);
\draw[dashed, style=dotted] (0.600000,2.100000) rectangle (0.900000, 1.800000);
\draw[dashed, style=dotted] (0.600000,1.800000) rectangle (0.900000, 1.500000);
\draw[dashed, style=dotted] (0.600000,1.500000) rectangle (0.900000, 1.200000);
\draw[dashed, style=dotted] (0.600000,1.200000) rectangle (0.900000, 0.900000);
\draw[dashed, style=dotted] (0.600000,0.900000) rectangle (0.900000, 0.600000);
\draw[dashed, style=dotted] (0.600000,0.600000) rectangle (0.900000, 0.300000);
\draw[dashed, style=dotted] (0.600000,0.300000) rectangle (0.900000, 0.000000);
\draw[dashed, style=dotted] (0.900000,2.400000) rectangle (1.200000, 2.100000);
\draw[dashed, style=dotted] (0.900000,2.100000) rectangle (1.200000, 1.800000);
\draw[dashed, style=dotted] (0.900000,1.800000) rectangle (1.200000, 1.500000);
\draw[dashed, style=dotted] (0.900000,1.500000) rectangle (1.200000, 1.200000);
\draw[dashed, style=dotted] (0.900000,1.200000) rectangle (1.200000, 0.900000);
\draw[dashed, style=dotted] (0.900000,0.900000) rectangle (1.200000, 0.600000);
\draw[dashed, style=dotted] (0.900000,0.600000) rectangle (1.200000, 0.300000);
\draw[dashed, style=dotted] (0.900000,0.300000) rectangle (1.200000, 0.000000);
\draw[dashed, style=dotted] (1.200000,2.400000) rectangle (1.500000, 2.100000);
\draw[dashed, style=dotted] (1.200000,2.100000) rectangle (1.500000, 1.800000);
\draw[dashed, style=dotted] (1.200000,1.800000) rectangle (1.500000, 1.500000);
\draw[dashed, style=dotted] (1.200000,1.500000) rectangle (1.500000, 1.200000);
\draw[dashed, style=dotted] (1.200000,1.200000) rectangle (1.500000, 0.900000);
\draw[dashed, style=dotted] (1.200000,0.900000) rectangle (1.500000, 0.600000);
\draw[dashed, style=dotted] (1.200000,0.600000) rectangle (1.500000, 0.300000);
\draw[dashed, style=dotted] (1.200000,0.300000) rectangle (1.500000, 0.000000);
\draw[dashed, style=dotted] (1.500000,2.400000) rectangle (1.800000, 2.100000);
\draw[dashed, style=dotted] (1.500000,2.100000) rectangle (1.800000, 1.800000);
\draw[dashed, style=dotted] (1.500000,1.800000) rectangle (1.800000, 1.500000);
\draw[dashed, style=dotted] (1.500000,1.500000) rectangle (1.800000, 1.200000);
\draw[dashed, style=dotted] (1.500000,1.200000) rectangle (1.800000, 0.900000);
\draw[dashed, style=dotted] (1.500000,0.900000) rectangle (1.800000, 0.600000);
\draw[dashed, style=dotted] (1.500000,0.600000) rectangle (1.800000, 0.300000);
\draw[dashed, style=dotted] (1.500000,0.300000) rectangle (1.800000, 0.000000);
\draw[dashed, style=dotted] (1.800000,2.400000) rectangle (2.100000, 2.100000);
\draw[dashed, style=dotted] (1.800000,2.100000) rectangle (2.100000, 1.800000);
\draw[dashed, style=dotted] (1.800000,1.800000) rectangle (2.100000, 1.500000);
\draw[dashed, style=dotted] (1.800000,1.500000) rectangle (2.100000, 1.200000);
\draw[dashed, style=dotted] (1.800000,1.200000) rectangle (2.100000, 0.900000);
\draw[dashed, style=dotted] (1.800000,0.900000) rectangle (2.100000, 0.600000);
\draw[dashed, style=dotted] (1.800000,0.600000) rectangle (2.100000, 0.300000);
\draw[dashed, style=dotted] (1.800000,0.300000) rectangle (2.100000, 0.000000);
\draw[dashed, style=dotted] (2.100000,2.400000) rectangle (2.400000, 2.100000);
\draw[dashed, style=dotted] (2.100000,2.100000) rectangle (2.400000, 1.800000);
\draw[dashed, style=dotted] (2.100000,1.800000) rectangle (2.400000, 1.500000);
\draw[dashed, style=dotted] (2.100000,1.500000) rectangle (2.400000, 1.200000);
\draw[dashed, style=dotted] (2.100000,1.200000) rectangle (2.400000, 0.900000);
\draw[dashed, style=dotted] (2.100000,0.900000) rectangle (2.400000, 0.600000);
\draw[dashed, style=dotted] (2.100000,0.600000) rectangle (2.400000, 0.300000);
\draw[dashed, style=dotted] (2.100000,0.300000) rectangle (2.400000, 0.000000);
\draw[draw, fill=black!100] (0.000000,0.900000) rectangle (0.300000, 0.600000);
  \draw[draw, fill=black!100] (0.000000,2.100000) rectangle (0.300000, 1.800000);
  \draw[draw, fill=black!100] (0.600000,2.400000) rectangle (0.900000, 2.100000);
  \draw[draw, fill=black!100] (0.600000,1.800000) rectangle (0.900000, 1.500000);
  \draw[draw, fill=black!100] (0.300000,0.600000) rectangle (0.600000, 0.300000);
  \draw[draw, fill=black!100] (0.300000,2.100000) rectangle (0.600000, 1.800000);
  \draw[draw] (0.000000,0.000000) rectangle (2.400000, 2.400000);
\draw[below] (1.200000,0) node {$A^{(1,1)}$};
    \end{scope}
    \begin{scope}[yshift = 1.65cm, xshift = 6.1cm]
      \draw[dashed, style=dotted] (0.000000,2.400000) rectangle (0.300000, 2.100000);
\draw[dashed, style=dotted] (0.000000,2.100000) rectangle (0.300000, 1.800000);
\draw[dashed, style=dotted] (0.000000,1.800000) rectangle (0.300000, 1.500000);
\draw[dashed, style=dotted] (0.000000,1.500000) rectangle (0.300000, 1.200000);
\draw[dashed, style=dotted] (0.000000,1.200000) rectangle (0.300000, 0.900000);
\draw[dashed, style=dotted] (0.000000,0.900000) rectangle (0.300000, 0.600000);
\draw[dashed, style=dotted] (0.000000,0.600000) rectangle (0.300000, 0.300000);
\draw[dashed, style=dotted] (0.000000,0.300000) rectangle (0.300000, 0.000000);
\draw[dashed, style=dotted] (0.300000,2.400000) rectangle (0.600000, 2.100000);
\draw[dashed, style=dotted] (0.300000,2.100000) rectangle (0.600000, 1.800000);
\draw[dashed, style=dotted] (0.300000,1.800000) rectangle (0.600000, 1.500000);
\draw[dashed, style=dotted] (0.300000,1.500000) rectangle (0.600000, 1.200000);
\draw[dashed, style=dotted] (0.300000,1.200000) rectangle (0.600000, 0.900000);
\draw[dashed, style=dotted] (0.300000,0.900000) rectangle (0.600000, 0.600000);
\draw[dashed, style=dotted] (0.300000,0.600000) rectangle (0.600000, 0.300000);
\draw[dashed, style=dotted] (0.300000,0.300000) rectangle (0.600000, 0.000000);
\draw[dashed, style=dotted] (0.600000,2.400000) rectangle (0.900000, 2.100000);
\draw[dashed, style=dotted] (0.600000,2.100000) rectangle (0.900000, 1.800000);
\draw[dashed, style=dotted] (0.600000,1.800000) rectangle (0.900000, 1.500000);
\draw[dashed, style=dotted] (0.600000,1.500000) rectangle (0.900000, 1.200000);
\draw[dashed, style=dotted] (0.600000,1.200000) rectangle (0.900000, 0.900000);
\draw[dashed, style=dotted] (0.600000,0.900000) rectangle (0.900000, 0.600000);
\draw[dashed, style=dotted] (0.600000,0.600000) rectangle (0.900000, 0.300000);
\draw[dashed, style=dotted] (0.600000,0.300000) rectangle (0.900000, 0.000000);
\draw[dashed, style=dotted] (0.900000,2.400000) rectangle (1.200000, 2.100000);
\draw[dashed, style=dotted] (0.900000,2.100000) rectangle (1.200000, 1.800000);
\draw[dashed, style=dotted] (0.900000,1.800000) rectangle (1.200000, 1.500000);
\draw[dashed, style=dotted] (0.900000,1.500000) rectangle (1.200000, 1.200000);
\draw[dashed, style=dotted] (0.900000,1.200000) rectangle (1.200000, 0.900000);
\draw[dashed, style=dotted] (0.900000,0.900000) rectangle (1.200000, 0.600000);
\draw[dashed, style=dotted] (0.900000,0.600000) rectangle (1.200000, 0.300000);
\draw[dashed, style=dotted] (0.900000,0.300000) rectangle (1.200000, 0.000000);
\draw[dashed, style=dotted] (1.200000,2.400000) rectangle (1.500000, 2.100000);
\draw[dashed, style=dotted] (1.200000,2.100000) rectangle (1.500000, 1.800000);
\draw[dashed, style=dotted] (1.200000,1.800000) rectangle (1.500000, 1.500000);
\draw[dashed, style=dotted] (1.200000,1.500000) rectangle (1.500000, 1.200000);
\draw[dashed, style=dotted] (1.200000,1.200000) rectangle (1.500000, 0.900000);
\draw[dashed, style=dotted] (1.200000,0.900000) rectangle (1.500000, 0.600000);
\draw[dashed, style=dotted] (1.200000,0.600000) rectangle (1.500000, 0.300000);
\draw[dashed, style=dotted] (1.200000,0.300000) rectangle (1.500000, 0.000000);
\draw[dashed, style=dotted] (1.500000,2.400000) rectangle (1.800000, 2.100000);
\draw[dashed, style=dotted] (1.500000,2.100000) rectangle (1.800000, 1.800000);
\draw[dashed, style=dotted] (1.500000,1.800000) rectangle (1.800000, 1.500000);
\draw[dashed, style=dotted] (1.500000,1.500000) rectangle (1.800000, 1.200000);
\draw[dashed, style=dotted] (1.500000,1.200000) rectangle (1.800000, 0.900000);
\draw[dashed, style=dotted] (1.500000,0.900000) rectangle (1.800000, 0.600000);
\draw[dashed, style=dotted] (1.500000,0.600000) rectangle (1.800000, 0.300000);
\draw[dashed, style=dotted] (1.500000,0.300000) rectangle (1.800000, 0.000000);
\draw[dashed, style=dotted] (1.800000,2.400000) rectangle (2.100000, 2.100000);
\draw[dashed, style=dotted] (1.800000,2.100000) rectangle (2.100000, 1.800000);
\draw[dashed, style=dotted] (1.800000,1.800000) rectangle (2.100000, 1.500000);
\draw[dashed, style=dotted] (1.800000,1.500000) rectangle (2.100000, 1.200000);
\draw[dashed, style=dotted] (1.800000,1.200000) rectangle (2.100000, 0.900000);
\draw[dashed, style=dotted] (1.800000,0.900000) rectangle (2.100000, 0.600000);
\draw[dashed, style=dotted] (1.800000,0.600000) rectangle (2.100000, 0.300000);
\draw[dashed, style=dotted] (1.800000,0.300000) rectangle (2.100000, 0.000000);
\draw[dashed, style=dotted] (2.100000,2.400000) rectangle (2.400000, 2.100000);
\draw[dashed, style=dotted] (2.100000,2.100000) rectangle (2.400000, 1.800000);
\draw[dashed, style=dotted] (2.100000,1.800000) rectangle (2.400000, 1.500000);
\draw[dashed, style=dotted] (2.100000,1.500000) rectangle (2.400000, 1.200000);
\draw[dashed, style=dotted] (2.100000,1.200000) rectangle (2.400000, 0.900000);
\draw[dashed, style=dotted] (2.100000,0.900000) rectangle (2.400000, 0.600000);
\draw[dashed, style=dotted] (2.100000,0.600000) rectangle (2.400000, 0.300000);
\draw[dashed, style=dotted] (2.100000,0.300000) rectangle (2.400000, 0.000000);
\draw[draw, fill=black!100] (1.500000,1.200000) rectangle (1.800000, 0.900000);
  \draw[draw, fill=black!100] (0.000000,1.200000) rectangle (0.300000, 0.900000);
  \draw[draw, fill=black!100] (0.600000,1.200000) rectangle (0.900000, 0.900000);
  \draw[draw] (0.000000,0.000000) rectangle (2.400000, 2.400000);
\draw[below] (1.200000,0) node {$A^{(1,2)}$};
    \end{scope}
    \begin{scope}[yshift = -1.6cm, xshift = 3.3cm]
      \draw[dashed, style=dotted] (0.000000,2.400000) rectangle (0.300000, 2.100000);
\draw[dashed, style=dotted] (0.000000,2.100000) rectangle (0.300000, 1.800000);
\draw[dashed, style=dotted] (0.000000,1.800000) rectangle (0.300000, 1.500000);
\draw[dashed, style=dotted] (0.000000,1.500000) rectangle (0.300000, 1.200000);
\draw[dashed, style=dotted] (0.000000,1.200000) rectangle (0.300000, 0.900000);
\draw[dashed, style=dotted] (0.000000,0.900000) rectangle (0.300000, 0.600000);
\draw[dashed, style=dotted] (0.000000,0.600000) rectangle (0.300000, 0.300000);
\draw[dashed, style=dotted] (0.000000,0.300000) rectangle (0.300000, 0.000000);
\draw[dashed, style=dotted] (0.300000,2.400000) rectangle (0.600000, 2.100000);
\draw[dashed, style=dotted] (0.300000,2.100000) rectangle (0.600000, 1.800000);
\draw[dashed, style=dotted] (0.300000,1.800000) rectangle (0.600000, 1.500000);
\draw[dashed, style=dotted] (0.300000,1.500000) rectangle (0.600000, 1.200000);
\draw[dashed, style=dotted] (0.300000,1.200000) rectangle (0.600000, 0.900000);
\draw[dashed, style=dotted] (0.300000,0.900000) rectangle (0.600000, 0.600000);
\draw[dashed, style=dotted] (0.300000,0.600000) rectangle (0.600000, 0.300000);
\draw[dashed, style=dotted] (0.300000,0.300000) rectangle (0.600000, 0.000000);
\draw[dashed, style=dotted] (0.600000,2.400000) rectangle (0.900000, 2.100000);
\draw[dashed, style=dotted] (0.600000,2.100000) rectangle (0.900000, 1.800000);
\draw[dashed, style=dotted] (0.600000,1.800000) rectangle (0.900000, 1.500000);
\draw[dashed, style=dotted] (0.600000,1.500000) rectangle (0.900000, 1.200000);
\draw[dashed, style=dotted] (0.600000,1.200000) rectangle (0.900000, 0.900000);
\draw[dashed, style=dotted] (0.600000,0.900000) rectangle (0.900000, 0.600000);
\draw[dashed, style=dotted] (0.600000,0.600000) rectangle (0.900000, 0.300000);
\draw[dashed, style=dotted] (0.600000,0.300000) rectangle (0.900000, 0.000000);
\draw[dashed, style=dotted] (0.900000,2.400000) rectangle (1.200000, 2.100000);
\draw[dashed, style=dotted] (0.900000,2.100000) rectangle (1.200000, 1.800000);
\draw[dashed, style=dotted] (0.900000,1.800000) rectangle (1.200000, 1.500000);
\draw[dashed, style=dotted] (0.900000,1.500000) rectangle (1.200000, 1.200000);
\draw[dashed, style=dotted] (0.900000,1.200000) rectangle (1.200000, 0.900000);
\draw[dashed, style=dotted] (0.900000,0.900000) rectangle (1.200000, 0.600000);
\draw[dashed, style=dotted] (0.900000,0.600000) rectangle (1.200000, 0.300000);
\draw[dashed, style=dotted] (0.900000,0.300000) rectangle (1.200000, 0.000000);
\draw[dashed, style=dotted] (1.200000,2.400000) rectangle (1.500000, 2.100000);
\draw[dashed, style=dotted] (1.200000,2.100000) rectangle (1.500000, 1.800000);
\draw[dashed, style=dotted] (1.200000,1.800000) rectangle (1.500000, 1.500000);
\draw[dashed, style=dotted] (1.200000,1.500000) rectangle (1.500000, 1.200000);
\draw[dashed, style=dotted] (1.200000,1.200000) rectangle (1.500000, 0.900000);
\draw[dashed, style=dotted] (1.200000,0.900000) rectangle (1.500000, 0.600000);
\draw[dashed, style=dotted] (1.200000,0.600000) rectangle (1.500000, 0.300000);
\draw[dashed, style=dotted] (1.200000,0.300000) rectangle (1.500000, 0.000000);
\draw[dashed, style=dotted] (1.500000,2.400000) rectangle (1.800000, 2.100000);
\draw[dashed, style=dotted] (1.500000,2.100000) rectangle (1.800000, 1.800000);
\draw[dashed, style=dotted] (1.500000,1.800000) rectangle (1.800000, 1.500000);
\draw[dashed, style=dotted] (1.500000,1.500000) rectangle (1.800000, 1.200000);
\draw[dashed, style=dotted] (1.500000,1.200000) rectangle (1.800000, 0.900000);
\draw[dashed, style=dotted] (1.500000,0.900000) rectangle (1.800000, 0.600000);
\draw[dashed, style=dotted] (1.500000,0.600000) rectangle (1.800000, 0.300000);
\draw[dashed, style=dotted] (1.500000,0.300000) rectangle (1.800000, 0.000000);
\draw[dashed, style=dotted] (1.800000,2.400000) rectangle (2.100000, 2.100000);
\draw[dashed, style=dotted] (1.800000,2.100000) rectangle (2.100000, 1.800000);
\draw[dashed, style=dotted] (1.800000,1.800000) rectangle (2.100000, 1.500000);
\draw[dashed, style=dotted] (1.800000,1.500000) rectangle (2.100000, 1.200000);
\draw[dashed, style=dotted] (1.800000,1.200000) rectangle (2.100000, 0.900000);
\draw[dashed, style=dotted] (1.800000,0.900000) rectangle (2.100000, 0.600000);
\draw[dashed, style=dotted] (1.800000,0.600000) rectangle (2.100000, 0.300000);
\draw[dashed, style=dotted] (1.800000,0.300000) rectangle (2.100000, 0.000000);
\draw[dashed, style=dotted] (2.100000,2.400000) rectangle (2.400000, 2.100000);
\draw[dashed, style=dotted] (2.100000,2.100000) rectangle (2.400000, 1.800000);
\draw[dashed, style=dotted] (2.100000,1.800000) rectangle (2.400000, 1.500000);
\draw[dashed, style=dotted] (2.100000,1.500000) rectangle (2.400000, 1.200000);
\draw[dashed, style=dotted] (2.100000,1.200000) rectangle (2.400000, 0.900000);
\draw[dashed, style=dotted] (2.100000,0.900000) rectangle (2.400000, 0.600000);
\draw[dashed, style=dotted] (2.100000,0.600000) rectangle (2.400000, 0.300000);
\draw[dashed, style=dotted] (2.100000,0.300000) rectangle (2.400000, 0.000000);
\draw[draw, fill=black!100] (2.100000,0.900000) rectangle (2.400000, 0.600000);
  \draw[draw, fill=black!100] (1.200000,1.800000) rectangle (1.500000, 1.500000);
  \draw[draw] (0.000000,0.000000) rectangle (2.400000, 2.400000);
\draw[below] (1.200000,0) node {$A^{(2,1)}$};
    \end{scope}
    \begin{scope}[yshift = -1.6cm, xshift = 6.1cm]
      \draw[dashed, style=dotted] (0.000000,2.400000) rectangle (0.300000, 2.100000);
\draw[dashed, style=dotted] (0.000000,2.100000) rectangle (0.300000, 1.800000);
\draw[dashed, style=dotted] (0.000000,1.800000) rectangle (0.300000, 1.500000);
\draw[dashed, style=dotted] (0.000000,1.500000) rectangle (0.300000, 1.200000);
\draw[dashed, style=dotted] (0.000000,1.200000) rectangle (0.300000, 0.900000);
\draw[dashed, style=dotted] (0.000000,0.900000) rectangle (0.300000, 0.600000);
\draw[dashed, style=dotted] (0.000000,0.600000) rectangle (0.300000, 0.300000);
\draw[dashed, style=dotted] (0.000000,0.300000) rectangle (0.300000, 0.000000);
\draw[dashed, style=dotted] (0.300000,2.400000) rectangle (0.600000, 2.100000);
\draw[dashed, style=dotted] (0.300000,2.100000) rectangle (0.600000, 1.800000);
\draw[dashed, style=dotted] (0.300000,1.800000) rectangle (0.600000, 1.500000);
\draw[dashed, style=dotted] (0.300000,1.500000) rectangle (0.600000, 1.200000);
\draw[dashed, style=dotted] (0.300000,1.200000) rectangle (0.600000, 0.900000);
\draw[dashed, style=dotted] (0.300000,0.900000) rectangle (0.600000, 0.600000);
\draw[dashed, style=dotted] (0.300000,0.600000) rectangle (0.600000, 0.300000);
\draw[dashed, style=dotted] (0.300000,0.300000) rectangle (0.600000, 0.000000);
\draw[dashed, style=dotted] (0.600000,2.400000) rectangle (0.900000, 2.100000);
\draw[dashed, style=dotted] (0.600000,2.100000) rectangle (0.900000, 1.800000);
\draw[dashed, style=dotted] (0.600000,1.800000) rectangle (0.900000, 1.500000);
\draw[dashed, style=dotted] (0.600000,1.500000) rectangle (0.900000, 1.200000);
\draw[dashed, style=dotted] (0.600000,1.200000) rectangle (0.900000, 0.900000);
\draw[dashed, style=dotted] (0.600000,0.900000) rectangle (0.900000, 0.600000);
\draw[dashed, style=dotted] (0.600000,0.600000) rectangle (0.900000, 0.300000);
\draw[dashed, style=dotted] (0.600000,0.300000) rectangle (0.900000, 0.000000);
\draw[dashed, style=dotted] (0.900000,2.400000) rectangle (1.200000, 2.100000);
\draw[dashed, style=dotted] (0.900000,2.100000) rectangle (1.200000, 1.800000);
\draw[dashed, style=dotted] (0.900000,1.800000) rectangle (1.200000, 1.500000);
\draw[dashed, style=dotted] (0.900000,1.500000) rectangle (1.200000, 1.200000);
\draw[dashed, style=dotted] (0.900000,1.200000) rectangle (1.200000, 0.900000);
\draw[dashed, style=dotted] (0.900000,0.900000) rectangle (1.200000, 0.600000);
\draw[dashed, style=dotted] (0.900000,0.600000) rectangle (1.200000, 0.300000);
\draw[dashed, style=dotted] (0.900000,0.300000) rectangle (1.200000, 0.000000);
\draw[dashed, style=dotted] (1.200000,2.400000) rectangle (1.500000, 2.100000);
\draw[dashed, style=dotted] (1.200000,2.100000) rectangle (1.500000, 1.800000);
\draw[dashed, style=dotted] (1.200000,1.800000) rectangle (1.500000, 1.500000);
\draw[dashed, style=dotted] (1.200000,1.500000) rectangle (1.500000, 1.200000);
\draw[dashed, style=dotted] (1.200000,1.200000) rectangle (1.500000, 0.900000);
\draw[dashed, style=dotted] (1.200000,0.900000) rectangle (1.500000, 0.600000);
\draw[dashed, style=dotted] (1.200000,0.600000) rectangle (1.500000, 0.300000);
\draw[dashed, style=dotted] (1.200000,0.300000) rectangle (1.500000, 0.000000);
\draw[dashed, style=dotted] (1.500000,2.400000) rectangle (1.800000, 2.100000);
\draw[dashed, style=dotted] (1.500000,2.100000) rectangle (1.800000, 1.800000);
\draw[dashed, style=dotted] (1.500000,1.800000) rectangle (1.800000, 1.500000);
\draw[dashed, style=dotted] (1.500000,1.500000) rectangle (1.800000, 1.200000);
\draw[dashed, style=dotted] (1.500000,1.200000) rectangle (1.800000, 0.900000);
\draw[dashed, style=dotted] (1.500000,0.900000) rectangle (1.800000, 0.600000);
\draw[dashed, style=dotted] (1.500000,0.600000) rectangle (1.800000, 0.300000);
\draw[dashed, style=dotted] (1.500000,0.300000) rectangle (1.800000, 0.000000);
\draw[dashed, style=dotted] (1.800000,2.400000) rectangle (2.100000, 2.100000);
\draw[dashed, style=dotted] (1.800000,2.100000) rectangle (2.100000, 1.800000);
\draw[dashed, style=dotted] (1.800000,1.800000) rectangle (2.100000, 1.500000);
\draw[dashed, style=dotted] (1.800000,1.500000) rectangle (2.100000, 1.200000);
\draw[dashed, style=dotted] (1.800000,1.200000) rectangle (2.100000, 0.900000);
\draw[dashed, style=dotted] (1.800000,0.900000) rectangle (2.100000, 0.600000);
\draw[dashed, style=dotted] (1.800000,0.600000) rectangle (2.100000, 0.300000);
\draw[dashed, style=dotted] (1.800000,0.300000) rectangle (2.100000, 0.000000);
\draw[dashed, style=dotted] (2.100000,2.400000) rectangle (2.400000, 2.100000);
\draw[dashed, style=dotted] (2.100000,2.100000) rectangle (2.400000, 1.800000);
\draw[dashed, style=dotted] (2.100000,1.800000) rectangle (2.400000, 1.500000);
\draw[dashed, style=dotted] (2.100000,1.500000) rectangle (2.400000, 1.200000);
\draw[dashed, style=dotted] (2.100000,1.200000) rectangle (2.400000, 0.900000);
\draw[dashed, style=dotted] (2.100000,0.900000) rectangle (2.400000, 0.600000);
\draw[dashed, style=dotted] (2.100000,0.600000) rectangle (2.400000, 0.300000);
\draw[dashed, style=dotted] (2.100000,0.300000) rectangle (2.400000, 0.000000);
\draw[draw, fill=black!100] (2.100000,0.300000) rectangle (2.400000, 0.000000);
  \draw[draw, fill=black!100] (1.200000,1.200000) rectangle (1.500000, 0.900000);
  \draw[draw] (0.000000,0.000000) rectangle (2.400000, 2.400000);
\draw[below] (1.200000,0) node {$A^{(2,2)}$};
    \end{scope}
    \draw [decoration={brace, amplitude=10pt}, decorate,
    ultra thick, black] (8.7,4.05) -- (8.7,-1.55);

    \begin{scope}[yshift = 0cm, xshift = 9.4cm]
      \draw[draw] (0.000000,0.000000) rectangle (2.400000, 2.400000);
\draw[dashed, style=dotted] (0.000000,2.400000) rectangle (0.300000, 2.100000);
\draw[dashed, style=dotted] (0.000000,2.100000) rectangle (0.300000, 1.800000);
\draw[dashed, style=dotted] (0.000000,1.800000) rectangle (0.300000, 1.500000);
\draw[dashed, style=dotted] (0.000000,1.500000) rectangle (0.300000, 1.200000);
\draw[dashed, style=dotted] (0.000000,1.200000) rectangle (0.300000, 0.900000);
\draw[dashed, style=dotted] (0.000000,0.900000) rectangle (0.300000, 0.600000);
\draw[dashed, style=dotted] (0.000000,0.600000) rectangle (0.300000, 0.300000);
\draw[dashed, style=dotted] (0.000000,0.300000) rectangle (0.300000, 0.000000);
\draw[dashed, style=dotted] (0.300000,2.400000) rectangle (0.600000, 2.100000);
\draw[dashed, style=dotted] (0.300000,2.100000) rectangle (0.600000, 1.800000);
\draw[dashed, style=dotted] (0.300000,1.800000) rectangle (0.600000, 1.500000);
\draw[dashed, style=dotted] (0.300000,1.500000) rectangle (0.600000, 1.200000);
\draw[dashed, style=dotted] (0.300000,1.200000) rectangle (0.600000, 0.900000);
\draw[dashed, style=dotted] (0.300000,0.900000) rectangle (0.600000, 0.600000);
\draw[dashed, style=dotted] (0.300000,0.600000) rectangle (0.600000, 0.300000);
\draw[dashed, style=dotted] (0.300000,0.300000) rectangle (0.600000, 0.000000);
\draw[dashed, style=dotted] (0.600000,2.400000) rectangle (0.900000, 2.100000);
\draw[dashed, style=dotted] (0.600000,2.100000) rectangle (0.900000, 1.800000);
\draw[dashed, style=dotted] (0.600000,1.800000) rectangle (0.900000, 1.500000);
\draw[dashed, style=dotted] (0.600000,1.500000) rectangle (0.900000, 1.200000);
\draw[dashed, style=dotted] (0.600000,1.200000) rectangle (0.900000, 0.900000);
\draw[dashed, style=dotted] (0.600000,0.900000) rectangle (0.900000, 0.600000);
\draw[dashed, style=dotted] (0.600000,0.600000) rectangle (0.900000, 0.300000);
\draw[dashed, style=dotted] (0.600000,0.300000) rectangle (0.900000, 0.000000);
\draw[dashed, style=dotted] (0.900000,2.400000) rectangle (1.200000, 2.100000);
\draw[dashed, style=dotted] (0.900000,2.100000) rectangle (1.200000, 1.800000);
\draw[dashed, style=dotted] (0.900000,1.800000) rectangle (1.200000, 1.500000);
\draw[dashed, style=dotted] (0.900000,1.500000) rectangle (1.200000, 1.200000);
\draw[dashed, style=dotted] (0.900000,1.200000) rectangle (1.200000, 0.900000);
\draw[dashed, style=dotted] (0.900000,0.900000) rectangle (1.200000, 0.600000);
\draw[dashed, style=dotted] (0.900000,0.600000) rectangle (1.200000, 0.300000);
\draw[dashed, style=dotted] (0.900000,0.300000) rectangle (1.200000, 0.000000);
\draw[dashed, style=dotted] (1.200000,2.400000) rectangle (1.500000, 2.100000);
\draw[dashed, style=dotted] (1.200000,2.100000) rectangle (1.500000, 1.800000);
\draw[dashed, style=dotted] (1.200000,1.800000) rectangle (1.500000, 1.500000);
\draw[dashed, style=dotted] (1.200000,1.500000) rectangle (1.500000, 1.200000);
\draw[dashed, style=dotted] (1.200000,1.200000) rectangle (1.500000, 0.900000);
\draw[dashed, style=dotted] (1.200000,0.900000) rectangle (1.500000, 0.600000);
\draw[dashed, style=dotted] (1.200000,0.600000) rectangle (1.500000, 0.300000);
\draw[dashed, style=dotted] (1.200000,0.300000) rectangle (1.500000, 0.000000);
\draw[dashed, style=dotted] (1.500000,2.400000) rectangle (1.800000, 2.100000);
\draw[dashed, style=dotted] (1.500000,2.100000) rectangle (1.800000, 1.800000);
\draw[dashed, style=dotted] (1.500000,1.800000) rectangle (1.800000, 1.500000);
\draw[dashed, style=dotted] (1.500000,1.500000) rectangle (1.800000, 1.200000);
\draw[dashed, style=dotted] (1.500000,1.200000) rectangle (1.800000, 0.900000);
\draw[dashed, style=dotted] (1.500000,0.900000) rectangle (1.800000, 0.600000);
\draw[dashed, style=dotted] (1.500000,0.600000) rectangle (1.800000, 0.300000);
\draw[dashed, style=dotted] (1.500000,0.300000) rectangle (1.800000, 0.000000);
\draw[dashed, style=dotted] (1.800000,2.400000) rectangle (2.100000, 2.100000);
\draw[dashed, style=dotted] (1.800000,2.100000) rectangle (2.100000, 1.800000);
\draw[dashed, style=dotted] (1.800000,1.800000) rectangle (2.100000, 1.500000);
\draw[dashed, style=dotted] (1.800000,1.500000) rectangle (2.100000, 1.200000);
\draw[dashed, style=dotted] (1.800000,1.200000) rectangle (2.100000, 0.900000);
\draw[dashed, style=dotted] (1.800000,0.900000) rectangle (2.100000, 0.600000);
\draw[dashed, style=dotted] (1.800000,0.600000) rectangle (2.100000, 0.300000);
\draw[dashed, style=dotted] (1.800000,0.300000) rectangle (2.100000, 0.000000);
\draw[dashed, style=dotted] (2.100000,2.400000) rectangle (2.400000, 2.100000);
\draw[dashed, style=dotted] (2.100000,2.100000) rectangle (2.400000, 1.800000);
\draw[dashed, style=dotted] (2.100000,1.800000) rectangle (2.400000, 1.500000);
\draw[dashed, style=dotted] (2.100000,1.500000) rectangle (2.400000, 1.200000);
\draw[dashed, style=dotted] (2.100000,1.200000) rectangle (2.400000, 0.900000);
\draw[dashed, style=dotted] (2.100000,0.900000) rectangle (2.400000, 0.600000);
\draw[dashed, style=dotted] (2.100000,0.600000) rectangle (2.400000, 0.300000);
\draw[dashed, style=dotted] (2.100000,0.300000) rectangle (2.400000, 0.000000);
\draw[draw, fill=black!100] (0.000000,0.900000) rectangle (0.300000, 0.600000);
  \draw[draw, fill=black!100] (0.000000,2.100000) rectangle (0.300000, 1.800000);
  \draw[draw, fill=black!100] (0.600000,2.400000) rectangle (0.900000, 2.100000);
  \draw[draw, fill=black!100] (0.600000,1.800000) rectangle (0.900000, 1.500000);
  \draw[draw, fill=black!100] (0.300000,0.600000) rectangle (0.600000, 0.300000);
  \draw[draw, fill=black!100] (0.300000,2.100000) rectangle (0.600000, 1.800000);
  \draw[draw, fill=black!100] (1.500000,1.200000) rectangle (1.800000, 0.900000);
  \draw[draw, fill=black!100] (0.000000,1.200000) rectangle (0.300000, 0.900000);
  \draw[draw, fill=black!100] (0.600000,1.200000) rectangle (0.900000, 0.900000);
  \draw[draw, fill=black!100] (2.100000,0.900000) rectangle (2.400000, 0.600000);
  \draw[draw, fill=black!100] (1.200000,1.800000) rectangle (1.500000, 1.500000);
  \draw[draw, fill=black!100] (2.100000,0.300000) rectangle (2.400000, 0.000000);
  \draw[draw, fill=black!100] (1.200000,1.200000) rectangle (1.500000, 0.900000);
  
      \draw[below] (1.25, 0) node{{$A$}};
    \end{scope}
  \end{tikzpicture}
  \caption{ The sub-graphs sampled from the previous step are
    un-permuted and quilted together to form a graph sampled from the
    MAGM. }
  \label{fig:quilt3}
\end{figure}
\begin{algorithm}
  \caption{Sampling Algorithm of Multiplicative Attribute Graphs}\label{alg:magsample}
  \begin{algorithmic}[1]
    \Function {MAGSampleEdges}{
      $\Thetat, f(1), \ldots, f(n)$
    }
    \State $B \gets \max_i |Z_i|$
    \For {$k \gets 1$ to $B$}
    \For {$l \gets 1$ to $B$}
    \State $E^{(k,l)} \gets$ \Call{\;\;\;KPGMSample}{$\Thetat$}
    \label{alg:kpgmsample_call}
    \For {each $(x, y) \in E^{(k,l)}$}
    \If {$(i,j)$ such that
      $i \in D_k$, $j \in D_l$, and
      $x = \lambda_i$, $y = \lambda_j$
      exists}
    \State $E \gets E \cup \{ (i,j) \}$
    \EndIf
    \EndFor
    \EndFor
    \EndFor
    \State \textbf{return} $E$
    \EndFunction
  \end{algorithmic}
\end{algorithm}

\begin{theorem}[Correctness]
  \label{thm:correctness}
  Algorithm~\ref{alg:magsample} samples the entries of the adjacency
  matrix $A$ independently with
  $\PP\rbr{A_{ij} = 1 \mid \Thetat, \lambda_1, \ldots, \lambda_n} = Q_{ij}$.
\end{theorem}
\begin{proof}
  See Appendix \ref{sec:TechnicalProofs}
\end{proof}

\subsection{Time Complexity}

Since the expected running time of Algorithm~\ref{alg:kpgmsample} is
$O(\log_{2}(n) |E|)$, the expected time complexity of quilting is
clearly $O(B^2 \log_{2}(n) |E|)$. The key technical
challenge in order to establish the efficiency of our scheme is to show
that with high probability $B$ is small, ideally $O(\log_{2}(n))$.

\paragraph{Balanced Attributes}
\label{sec:BalancedAttributes}
Suppose the distribution of each attribute is balanced, that is,
$\mu^{(1)} = \mu^{(2)} = \cdots = \mu^{(d)} = 0.5$ and $n = 2^d$. Define
a random variable $X_{c}^{i} = 1$ if $\lambda_i = c$ and zero
otherwise. Since $\mu^{(k)} = 0.5$, it follows that $\PP\rbr{X_{c}^{i} =
  1} = \frac{1}{2^{d}} = \frac{1}{n}$. If we let $Y_c = \sum_{i=1}^{n}
X_{c}^{i}$, then clearly $B = \max_c Y_c$.  Since $X_{c}^{i}$ are
independent Bernoulli random variables, $Y_c$ is a binomial random
variable which converges to a Poisson random variable with parameter $1$
as $n \to \infty$.  Using standard Chernoff bounds for the Poisson
distribution (see Theorem~\ref{thm:chernoff}), we can write
\begin{align}
  \label{eq:poisson-chernoff}
  \PP\rbr{Y_c > t} & \leq \frac{e^t}{e t^t}, \text{ and hence } \\
  \PP\rbr{ B = \max Y_c > t} & \leq \sum_{c=1}^n \PP[ Y_c > t] \leq \frac{n
    e^t} {e t^t}.
\end{align}
Replacing $t$ by $\log_{2}(n)$,
\begin{align}
  \label{eq:B_bound}
  \PP\rbr{B > \log_{2}(n)} \leq \frac{n^2} {e (\log_{2}(n))^{\log_{2}(n)}} .
\end{align}
As $n \rightarrow \infty$, \eqref{eq:B_bound} goes to 0 (also see
Figure~\ref{fig:binnum_05}). Therefore we have
\begin{theorem}
  \label{thm:scalability}
  When $\mu^{(1)} = \mu^{(2)} = \cdots = \mu^{(d)} = 0.5$, and $n =
  2^d$, with high probability the size of partitions $B$ is smaller than
  $\log_{2}(n)$.
\end{theorem}

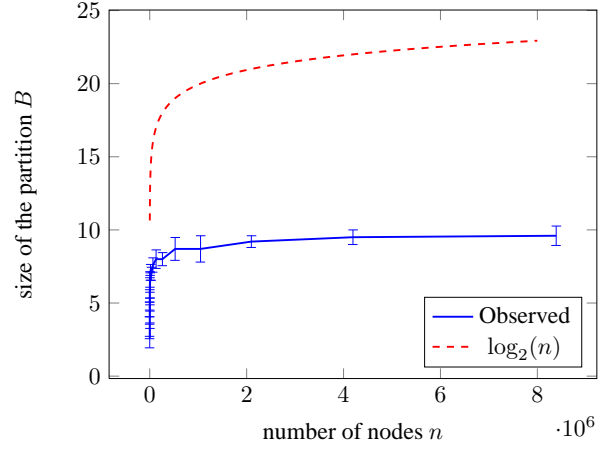
\begin{figure}
  \begin{tikzpicture}[scale=0.9]
    \begin{axis}[
      legend pos=south east,
      width=250pt, height=200pt, no markers,
      xlabel={number of nodes $n$},
      ylabel={size of the partition $B$},
      domain=0:8000000,
      samples=5000
      ]

      \addplot+[
      color=blue, mark=x, thick,
      error bars/.cd,
      y dir=both, y explicit,
      error mark=-]
      table[x=x,y=y,y error=errory]
      {partnums_0.50.dat};

      \addplot+[color=red, no markers, thick, dashed]
      {ln(x)/ln(2)};

      \legend{Observed, $\log_{2}(n)$}
    \end{axis}
  \end{tikzpicture}
  \caption{ Number of nodes vs.\ size of the partition, when
    $\mu^{(k)}$'s are all set to be 0.5. For each $n$, we performed 10
    trials and report average values (blue solid line). The red dashed
    line is the bound predicted by \eqref{eq:B_bound}. Observe that in
    practice, the size of the partition grows much slower than
    $O(\log_{2}(n))$. }
  \label{fig:binnum_05}
\end{figure}

\paragraph{Unbalanced Attributes}
\label{sec:UnbalancedAttributes}
As before, we let $\mu^{(1)} = \mu^{(2)} = \cdots = \mu^{(d)} = \mu$ and
$n = 2^d$, but now we analyze the case when $\mu \neq 0.5$. By
transposing some of $\Theta^{(k)}$ if necessary, it suffices to restrict
our attention to $\mu \in (0.5, 1]$. We define the random variables
$X_{c}^{i}$ and $Y_c$ as in the previous section. However, now
$\PP\rbr{X_{c}^{i}}$ depends on the number of 1's in the binary
representation of $c$.  In particular, if $c=2^d=n$ then
$\PP\rbr{X_{n}^{i}} = \mu^{\log_{2}(n)}$ and
$Y_n=n\mu^{\log_{2}(n)}$. Furthermore, $\PP\rbr{X_{c}^{i}} < \mu^{\log_{2}(n)}$
for every $c \neq n$. Therefore, when $\mu$ is close to $1$ and $n$ is
large, $B := \max_c Y_c$ equals $Y_n = n \mu^{\log_{2}(n)}$ with high
probability (see Figure~\ref{fig:binnums}). The expected running time of
our algorithm then becomes $(n^{\log_{2} (\mu) + 1} \log_{2}(n) |E|)$.

\begin{figure}
  \begin{tikzpicture}[scale=0.49]
    \begin{axis}[
      legend pos=south east,
      width=250pt, height=200pt, no markers,
      xlabel={number of nodes}, ylabel={size of the partition},
      title={{\Large $\mu = 0.55$}}
      ]

      \addplot+[
      color=blue, mark=x, thick,
      error bars/.cd,
      y dir=both, y explicit,
      error mark=-]
      table[x=x,y=y,y error=errory]
      {partnums_0.55.dat};

      \addplot+[
      color=red, no markers, thick, dashed]
      table[x=x,y=log2n]
      {partnums_0.55_theory.dat};

      \addplot+[
      color=black, no markers, thick, dashed]
      table[x=x,y=nmud]
      {partnums_0.55_theory.dat};

      \legend{Observed, $\log_{2}(n)$, $n\mu^{d}$}
    \end{axis}
  \end{tikzpicture}
  \begin{tikzpicture}[scale=0.49]
    \begin{axis}[
      legend pos=south east,
      width=250pt, height=200pt, no markers,
      xlabel={number of nodes}, 
      title={{\Large $\mu = 0.60$}}
      ]

      \addplot+[
      color=blue, mark=x, thick,
      error bars/.cd,
      y dir=both, y explicit,
      error mark=-]
      table[x=x,y=y,y error=errory]
      {partnums_0.60.dat};

      \addplot+[
      color=red, no markers, thick, dashed]
      table[x=x,y=log2n]
      {partnums_0.60_theory.dat};

      \addplot+[
      color=black, no markers, thick, dashed]
      table[x=x,y=nmud]
      {partnums_0.60_theory.dat};

      \legend{Observed, $\log_{2}(n)$, $n\mu^{d}$}
    \end{axis}
  \end{tikzpicture}

  \begin{tikzpicture}[scale=0.49]
    \begin{axis}[
      legend pos=south east,
      width=250pt, height=200pt, no markers,
      xlabel={number of nodes}, 
      ylabel={size of the partition},
      title={{\Large $\mu = 0.70$}}
      ]

      \addplot+[
      color=blue, mark=x, thick,
      error bars/.cd,
      y dir=both, y explicit,
      error mark=-]
      table[x=x,y=y,y error=errory]
      {partnums_0.70.dat};

      \addplot+[
      color=red, no markers, thick, dashed]
      table[x=x,y=log2n]
      {partnums_0.70_theory.dat};

      \addplot+[
      color=black, no markers, thick, dashed]
      table[x=x,y=nmud]
      {partnums_0.70_theory.dat};

      \legend{Observed, $\log_{2}(n)$, $n\mu^{d}$}
    \end{axis}
  \end{tikzpicture}
  \begin{tikzpicture}[scale=0.49]
    \begin{axis}[
      legend pos=south east,
      width=250pt, height=200pt, no markers,
      xlabel={number of nodes},
      title={{\Large $\mu = 0.90$}}
      ]

      \addplot+[
      color=blue, mark=x, thick,
      error bars/.cd,
      y dir=both, y explicit,
      error mark=-]
      table[x=x,y=y,y error=errory]
      {partnums_0.90.dat};

      \addplot+[
      color=red, no markers, thick, dashed]
      table[x=x,y=log2n]
      {partnums_0.90_theory.dat};

      \addplot+[
      color=black, no markers, thick, dashed]
      table[x=x,y=nmud]
      {partnums_0.90_theory.dat};

      \legend{Observed, $\log_{2}(n)$, $n\mu^{d}$}
    \end{axis}
  \end{tikzpicture}
  \caption{Number of nodes vs. size of the partition,
    when $\mu^{(k)}$'s are all set to be 0.55, 0.60, 0.70,
    and 0.90. Again, 10 number of $F$ matrices were
    sampled for each $n$ and $\mu$, and the average
    size of the partition is taken.
    For small values of $n$, $n \mu^d$ approximation is not
    tight but the observed value is sandwiched between
    $\log_{2}(n)$ and $n \mu^d$. For $\mu > 0.70$,
    the $n\mu^d$ approximation is tight.}
  \label{fig:binnums}
\end{figure}
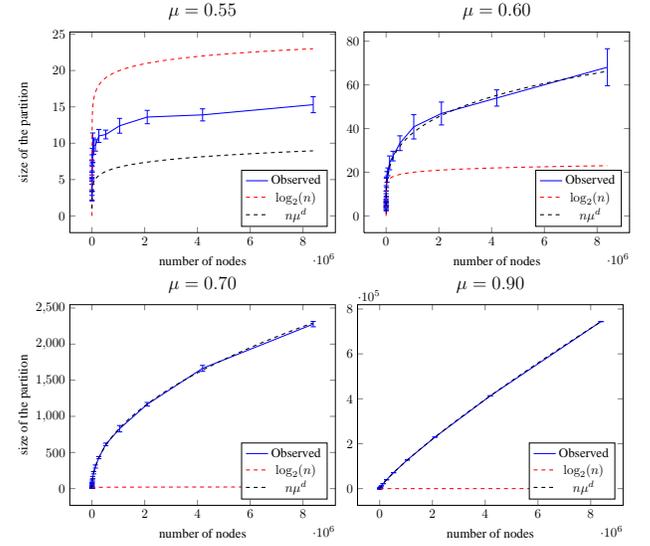

\subsection{Handling the Case When $n \neq 2^d$}
\label{sec:case_logn_neq_2d}

To simplify the analysis assume $\mu^{(1)} = \cdots = \mu^{(k)} =
0.5$.

\paragraph{$n > 2^d$ Case}
First, consider the case $n > 2^d$, and let $d' := \lceil \log_{2}(n)
\rceil$, $d'' := \lfloor \log_{2}(n) \rfloor$. Each $Y_c$ now converges
to the Poisson distribution with parameter $\frac n {2^d}$, and using
the Chernoff bound again, one can prove that $B = O(2^{d' - d}
\log_{2}(n))$ with high probability, for large $n$.  For details, refer
to Appendix~\ref{sec:B_upperbound}.

Therefore, the expected time complexity is bounded by $O((\log_{2}(n))^2
d |E|)$, and the algorithm gets faster as $d$ decreases.

\paragraph{$n < 2^d$ Case}
For the sake of completeness we will make some remarks for the case when
$n < 2^d$. However, \citet{KimLes10} and \citet{KimLes11} report that in
practice $d \approx \log_{2}(n)$ usually results in the most realistic
graphs. In general, for large $d$ the value of $B$ is small, but the
number of edges in graphs sampled by Algorithm~\ref{alg:kpgmsample}
,which is called by line~\ref{alg:kpgmsample_call} of
Algorithm~\ref{alg:magsample}, increases exponentially with
$d$. Therefore, the overall complexity of the algorithm is at least
$\Omega( 4^{d - d''} \EE[|E|])$, and thus naively applying
Algorithm~\ref{alg:magsample} is not advantageous when $d - d''$ is
high. Our experiments in Section~\ref{sec:Effectd} confirm this
behavior. 

\section{Speeding up the Algorithm}
\label{sec:speed_up}

The key to speeding up our algorithm for the case when $\mu^{(k)} \neq
0.5$ is the following observation: When $\mu^{(k)}$ approaches 0 or 1,
the number of distinct attribute configurations reduces
significantly. For instance, when $\mu^{(k)}$ approaches 1 attribute
configurations which contain more 1's in their binary representation are
generated with greater frequency. Similarly, when $\mu^{(k)}$ approaches
0 the attribute configurations which contain more 0's in their binary
representation are preferred. Figure~\ref{fig:color_freq} represents
this phenomenon visually.

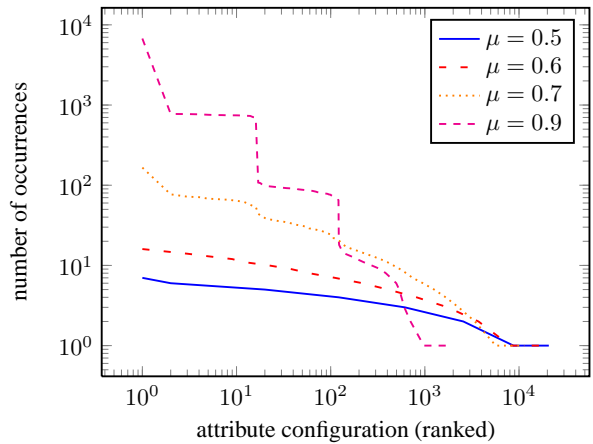
\begin{figure}
  \begin{tikzpicture}[scale=0.9]
    \begin{loglogaxis}[
      legend pos=north east,
      width=250pt, height=200pt, 
      xlabel={attribute configuration (ranked)},
      ylabel={number of occurrences},
      thick, no markers
      ]

      \addplot[solid, color = blue]
      table[x=xs,y=ys]{color_freq_0.5.txt};
      \addplot[loosely dashed, color=red]
      table[x=xs,y=ys]{color_freq_0.6.txt};
      \addplot[dotted, color=orange]
      table[x=xs,y=ys]{color_freq_0.7.txt};
      \addplot[dashed, color=magenta]
      table[x=xs,y=ys]{color_freq_0.9.txt};

      \legend{$\mu = 0.5$, $\mu = 0.6$, $\mu = 0.7$, $\mu = 0.9$}
    \end{loglogaxis}
  \end{tikzpicture}
  \caption{We rank attribute configurations based on their frequency of
    occurrence and plot them for different values of $\mu$. We fixed $d
    = 15$, and $n = 2^{15}$ for this plot. When $\mu = 0.5$, the graph
    is very flat since every attribute configuration has the same
    probability $\frac 1 {2^d}$ of being sampled. On the other hand,
    when $\mu = 0.9$, the probability mass is very concentrated on a
    small number of configurations. Note that this is a log-log plot.}
  \label{fig:color_freq}
\end{figure}

We select a number $B'$ (see below) and collect all nodes $i$ whose
attribute configuration $\lambda_i$ occurs at most $B'$ times in the set
$\cbr{\lambda_1, \lambda_2, \ldots, \lambda_n}$ into a set $W$. Since
each attribute configuration occurs at most $B'$ times in $W$, the
sub-graph corresponding to the nodes in $W$ can be sampled in
$O\rbr{B'^2 \log_{2}(n) \abr{E}}$ by using Algorithm
\ref{alg:magsample}. 

We partition the nodes whose attribute configuration occurs more than
$B'$ times into sets $\hat{D}_{1}, \ldots, \hat{D}_{R}$ such that the
attribute configuration of each node in $\hat{D}_{i}$ is the same, say
$\lambda'_i$. The sub-graph corresponding to each $\hat{D}_{i}$ is an
uniform random graph with probability of an edge being equal to
$P_{\lambda'_i, \lambda'_i}$. On the other hand, the sub-graph
corresponding to nodes $\hat{D}_{i}$ and $\hat{D}_{j}$ for $i \neq j$ is
also an uniform random graph with the probability of an edge being equal
to $P_{\lambda'_i, \lambda'_j}$. Finally, the sub-graph corresponding to a
node $i'$ in $W$ and the set $\hat{D}_{j}$ is an uniform random graph
with the probability of an edge being equal to $P_{\lambda_{i'},
  \lambda'_j}$. These sub-graphs can be sampled with $O((|W|+d)R+|E|)$, 
$O(dR + |E|)$, and $O(dR^2 + |E|)$ effort 
respectively\footnote{ Instead of sampling $k$
  i.i.d.\ Bernoulli random variables $X_1, X_2, \ldots, X_k$ with
  parameter $p$, we use a geometric distribution with parameter $p$ to
  sample to generate random variables $K_j$ such that $1 \leq K_1 < K_2
  < \ldots \leq k$, and set $X_{K_{j}} = 1$.}.

It remains to discuss how to choose the parameter $B'$.  Towards this
end let
\begin{align*}
  T(B') = B'^2 \log(n) \abr{E} + (|W|+d)R + d R^2.
\end{align*}
Then, the overall time complexity of our algorithm is $O(T(B'))$.
We calculate $T(B')$ for every $B'$,
and choose the value which minimizes $T(B')$.
Since there are only $n$ distinct values of $B'$,
this procedure requires $O(n)$ time.

\section{Experiments}
\label{sec:experiment}

We empirically evaluated the efficiency and scalability of our sampling
algorithm. Our experiments are designed to answer the following
questions: Does our algorithm produce graphs with similar
characteristics as observed by \citet{KimLes10}. How does our algorithm
scale as a function of $n$, the number of nodes in the
graph. Furthermore, since our theoretical analysis assumed $\mu=0.5$ and
$d = \log_{2}(n)$ we were interested in the following additional questions:
How does the algorithm behave for $\mu \neq 0.5$. How does our algorithm
scale when the number of features $d$ is different from $\log_{2}(n)$.

Our code is implemented in C++ and will be made available for download
from \url{http://www.stat.purdue.edu/~yun3}. All experiments are run on a machine with
a 2.1 GHz processor 
running Linux. For the first three
experiments we uniformly set $n = 2^d$, where $n$ is the number of nodes
in the graph and $d$ is the dimension of the features. We used the same
$\Theta$ matrices at all levels, that is, we set $\Theta = \Theta^{(1)}
= \Theta^{(2)} = \cdots = \Theta^{(k)}$.  Furthermore, we experimented
with the following $\Theta$ matrices used by \citet{KimLes10} and
\citet{MorNev09} respectively:
\begin{align}
  \Theta_1 =
  \mymatrix{cc}{
    0.15 & 0.7 \\
    0.7 & 0.85
  } \text{ and }
  \Theta_2 =
  \mymatrix{cc}{
    0.35 & 0.52 \\
    0.52 & 0.95
  }
\end{align}

\subsection{Properties of the Generated Graphs}
\label{sec:PropGenerGraphs}


Theorem~\ref{thm:correctness} guarantees that our algorithm generates
valid graphs from the MAGM. In our first experiment we also verify this
claim empirically. Towards this end, we set $\mu = 0.5$ and generated
graphs of size $n=2^d$ for various values of $d$. For each $n$ we
repeated the sampling procedure 10 times and studied various properties
of the generated graphs. 
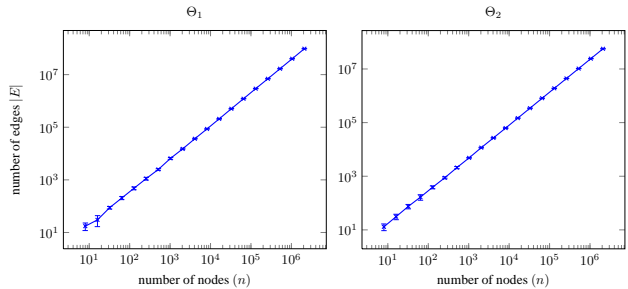
\begin{figure}
  \begin{center}
    \begin{tikzpicture}[scale=0.51]

      \begin{loglogaxis}[
        title={$\Theta_1$}, xlabel={number of nodes $(n)$},
        ylabel={number of edges $|E|$},
        cycle list name=black white]

        \addplot+[
        color=blue, mark=x, thick,
        error bars/.cd,
        y dir=both, y explicit,
        error mark=-]
        table[x=x,y=y,y error=errory]
        {cluster_prop_numedges_1.dat};

      \end{loglogaxis}
    \end{tikzpicture}
    \begin{tikzpicture}[scale=0.51]
      \begin{loglogaxis}[
        title={$\Theta_2$}, xlabel={number of nodes $(n)$},
        cycle list name=black white]

        \addplot+[
        color=blue, mark=x, thick,
        error bars/.cd,
        y dir=both, y explicit,
        error mark=-]
        table[x=x,y=y,y error=errory]
        {cluster_prop_numedges_2.dat};

      \end{loglogaxis}
    \end{tikzpicture}
  \end{center}
  \caption{The number of edges $|E|$ as a function of the size $n$ of
    the graphs sampled from the MAGM for two different values of
    $\Theta$. The near linear rate of growth on the log-log plots
    confirms the observation that $|E| = n^c$ for some constant $c$. }
  \label{fig:num_edges}
\end{figure}
As reported by \citet{KimLes10}, the number of
edges $|E|$ in the graphs generated by MAGM grow as $|E| = n^c$ for some
constant $c$. Graph samples generated by our algorithm also confirm to
this observation, as can be seen in
Figure~\ref{fig:num_edges}. Furthermore, \citet{KimLes10} report that
the proportion of nodes in the largest strong component increases
asymptotically to $1$. We also observe this behavior in the samples
generated by our algorithm (see
Figure~\ref{fig:strong_component}). These experiments indeed confirm
that our algorithm samples valid graphs from the MAGM.

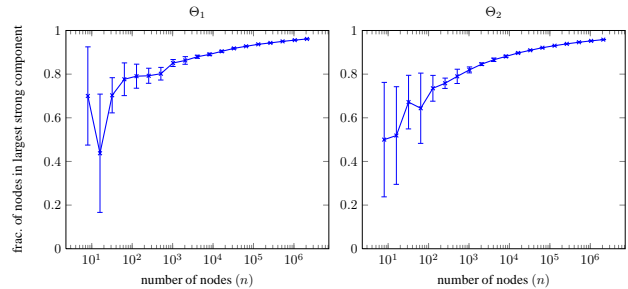
\begin{figure}
  \begin{center}
    \begin{tikzpicture}[scale=0.51]
      \begin{semilogxaxis}[
        ymin = 0, ymax = 1,
        title={$\Theta_1$}, xlabel={number of nodes $(n)$},
        ylabel={frac.\ of nodes in largest strong component},
        cycle list name=black white]

        \addplot+[
        color=blue, mark=x, thick,
        error bars/.cd,
        y dir=both, y explicit,
        error mark=-]
        table[x=x,y=y,y error=errory]
        {cluster_prop_gscc_1.dat};

      \end{semilogxaxis}
    \end{tikzpicture}
    \begin{tikzpicture}[scale=0.51]
      \begin{semilogxaxis}[
        ymin=0, ymax=1,
        title={$\Theta_2$}, xlabel={number of nodes $(n)$},
        cycle list name=black white]

        \addplot+[
        color=blue, mark=x, thick,
        error bars/.cd,
        y dir=both, y explicit,
        error mark=-]
        table[x=x,y=y,y error=errory]
        {cluster_prop_gscc_2.dat};

      \end{semilogxaxis}
    \end{tikzpicture}
  \end{center}
  \caption{The fraction of nodes in the largest strong component as a
    function of the size $n$ of the graphs sampled from the MAGM for two
    different values of $\Theta$. Asymptotically, the fraction of edges
    approaches $1$ implying that the entire graph is part of the same
    strong component. }
  \label{fig:strong_component}
\end{figure}

\subsection{Scalability}
\label{sec:Scalability}

To study the scalability of our algorithm we fixed $\mu = 0.5$ and
generated 10 graphs of size $n=2^d$ for various values of $d$.
Figure~\ref{fig:compare_naive} compares the running time of our
algorithm vs a naive scheme which uses $n^2$ independent Bernoulli
trials based on the entries of the edge probability matrix.

Note that using the naive sampling scheme we could not sample graphs
with more than 262,144 nodes in less than 8 hours. In contrast, the
running time of our algorithm grows significantly slower than $O(n^2)$
and consequently we were able to comfortably sample graphs with a
million nodes in less than twenty minutes. The largest graphs produced
by our algorithm contain over 8 million nodes (8,388,608) and 20 billion
edges.  In fact these graphs are, to the best of our knowledge, at least
32 times larger than the largest MAGM graphs reported in literature in
terms of number of nodes.  
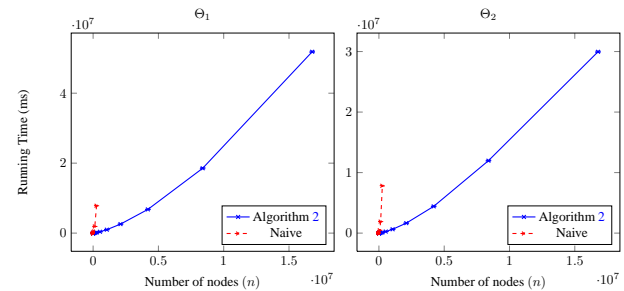
\begin{figure}
  \begin{center}
    \begin{tikzpicture}[scale=0.51]
      \begin{axis}[title=$\Theta_1$,
        legend pos=south east,
        xlabel={Number of nodes $(n)$},
        ylabel={Running Time (ms)}]

        \addplot+[
        color=blue, mark=x, thick,
        error bars/.cd,
        y dir=both, y explicit,
        error mark=-]
        table[x=x,y=y,y error=errory]
        {cluster_mu50_1.dat};

        \addplot+[
        color=red, mark=x, thick, dashed,
        error bars/.cd,
        y dir=both, y explicit,
        error mark=-]
        table[x=x,y=y,y error=errory]
        {cluster_mu50_naive_1.dat};
        \legend{Algorithm~\ref{alg:magsample}, Naive}
      \end{axis}
    \end{tikzpicture}
    \begin{tikzpicture}[scale=0.51]
      \begin{axis}[title=$\Theta_2$,
        legend pos=south east,
        xlabel={Number of nodes $(n)$}
        ]

        \addplot+[
        color=blue, mark=x, thick,
        error bars/.cd,
        y dir=both, y explicit,
        error mark=-]
        table[x=x,y=y,y error=errory]
        {cluster_mu50_2.dat};

        \addplot+[
        color=red, mark=x, thick, dashed,
        error bars/.cd,
        y dir=both, y explicit,
        error mark=-]
        table[x=x,y=y,y error=errory]
        {cluster_mu50_naive_2.dat};
        \legend{Algorithm~\ref{alg:magsample}, Naive}
      \end{axis}
    \end{tikzpicture}
  \end{center}
  \caption{ Comparison of running time (in milliseconds) of our
    algorithm vs the naive sampling scheme as a function of the size $n$
    of the graphs sampled from the MAGM for two different values of
    $\Theta$. }
  \label{fig:compare_naive}
\end{figure}
Furthermore, we observed that our algorithm
exhibits the same behavior across a range of $\Theta$ values (not
reported here). To further demonstrate the scalability of our algorithm
we plot the running time of our algorithm normalized by the number of
edges in the graph in Figure~\ref{fig:compare_peredge}. Across a range
of $n$ values, our algorithm spends a constant time for each edge that
is generated. We can therefore conclude that the running time of our
algorithm grows empirically as $O(|E|)$.
\begin{figure}
  \begin{center}

    \begin{tikzpicture}[scale=0.51]
      \begin{axis}[title=$\Theta_1$,
        legend pos=north east,
        xlabel={Number of nodes $(n)$},
        ylabel={Running time per edge}
        ]

        \addplot+[
        color=blue, mark=x, thick,
        error bars/.cd,
        y dir=both, y explicit,
        error mark=-]
        table[x=x,y=y,y error=errory]
        {cluster_mu50_peredge_1.dat};

        \addplot+[
        color=red, mark=x, thick, dashed,
        error bars/.cd,
        y dir=both, y explicit,
        error mark=-]
        table[x=x,y=y,y error=errory]
        {cluster_mu50_naive_peredge_1.dat};
        \legend{Algorithm~\ref{alg:magsample}, Naive}
      \end{axis}
    \end{tikzpicture}
    \begin{tikzpicture}[scale=0.51]
      \begin{axis}[title=$\Theta_2$,
        legend pos=north east,
        xlabel={Number of nodes $(n)$}
        ]

        \addplot+[
        color=blue, mark=x, thick,
        error bars/.cd,
        y dir=both, y explicit,
        error mark=-]
        table[x=x,y=y,y error=errory]
        {cluster_mu50_peredge_2.dat};

        \addplot+[color=red, mark=x, thick, dashed,
        error bars/.cd,
        y dir=both, y explicit,
        error mark=-]
        table[x=x,y=y,y error=errory]
        {cluster_mu50_naive_peredge_2.dat};
        \legend{Algorithm~\ref{alg:magsample}, Naive}
      \end{axis}
    \end{tikzpicture}
  \end{center}
  \caption{Running time per each edge (in milliseconds) for our
    algorithms vs the naive algorithm as a function of $n$ the number of
    nodes. Note that the normalized running time of our algorithm is
    nearly constant and does not change as $n$ increases.}
  \label{fig:compare_peredge}
\end{figure}
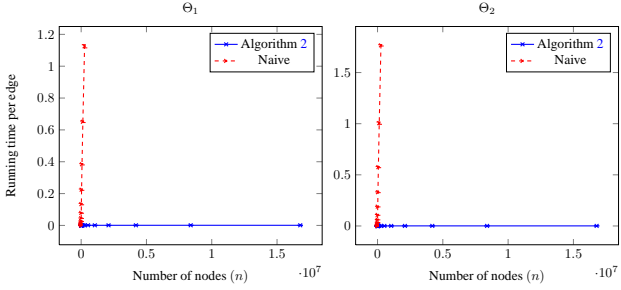

\subsection{Effect of $\mu$}
\label{sec:Effectmu}

Our theoretical analysis (Theorem~\ref{thm:scalability}) guarantees the
scalability of our sampling algorithm for $\mu = 0.5$. In this section
we explore empirically how the running time of our algorithm varies as
we vary $\mu$. Towards this end we define and study the relative running
time
$\rho(\mu) := \frac{T(\mu)} {T(0.5)}$,
where $T(\mu)$ denotes the running time of the algorithm as a function
of $\mu$. In Figure~\ref{fig:effect_mu} we plot $\rho(\mu)$ for
different values of $n=2^d$.

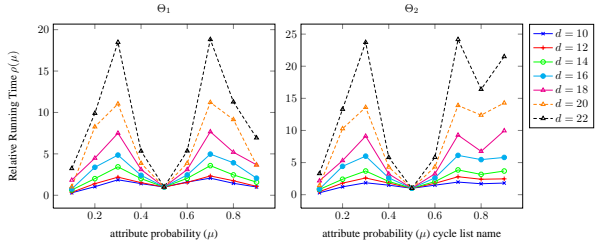
\begin{figure}
  \begin{center}
    \begin{tikzpicture}[scale=0.43]
      \begin{axis}[
        legend style={
          cells={anchor=east},
          legend pos=outer north east,
        },
        title={$\Theta_1$}, xlabel={attribute probability
          ($\mu$)}, ylabel={Relative Running Time $\rho({\mu})$},
        cycle list name=black white]
        \addplot+[color=blue,mark=x,thick] table[x=x,y=y]
        {mus_d10_1.dat};
        \addplot+[color=red,mark=+,thick] table[x=x,y=y]
        {mus_d12_1.dat};
        \addplot+[color=green,mark=o,thick] table[x=x,y=y]
        {mus_d14_1.dat};
        \addplot+[color=cyan,mark=*,thick] table[x=x,y=y]
        {mus_d16_1.dat};
        \addplot+[color=magenta,mark=triangle,thick] table[x=x,y=y]
        {mus_d18_1.dat};
        \addplot+[color=orange,mark=triangle,thick] table[x=x,y=y]
        {mus_d20_1.dat};
        \addplot+[color=black,mark=triangle,thick] table[x=x,y=y]
        {mus_d22_1.dat};
      \end{axis}
    \end{tikzpicture}
    \begin{tikzpicture}[scale=0.43]
      \begin{axis}[
        legend style={
          cells={anchor=east},
          legend pos=outer north east,
        },
        title={$\Theta_2$}, xlabel={attribute probability
          ($\mu$)}
        cycle list name=black white]
          \addplot+[color=blue,mark=x,thick] table[x=x,y=y]
          {mus_d10_2.dat};
          \addplot+[color=red,mark=+,thick] table[x=x,y=y]
          {mus_d12_2.dat};
          \addplot+[color=green,mark=o,thick] table[x=x,y=y]
          {mus_d14_2.dat};
          \addplot+[color=cyan,mark=*,thick] table[x=x,y=y]
          {mus_d16_2.dat};
          \addplot+[color=magenta,mark=triangle,thick] table[x=x,y=y]
          {mus_d18_2.dat};
          \addplot+[color=orange,mark=triangle,thick] table[x=x,y=y]
          {mus_d20_2.dat};
          \addplot+[color=black,mark=triangle,thick] table[x=x,y=y]
          {mus_d22_2.dat};
          \legend{$d=10$, $d=12$, $d=14$, $d=16$, $d=18$,
          $d=20$, $d=22$}
      \end{axis}
    \end{tikzpicture}
  \end{center}
  \caption{Relative running time $\rho(\mu)$ for two different values of
    $\Theta$.}
  \label{fig:effect_mu}
\end{figure}

As expected, our algorithm performs well for $\mu = 0.5$, 
in which case the size of partition $B$ is bounded by $\log_{2}(n)$ with high
probability.
Similarly, when $\mu \approx 0$ or $1$ the attribute
configurations have significantly less diversity and hence sampling
becomes easy. However, there is also a tendency that the running time
increases as $\mu$ increases. This is because the number of edges is
also a function of $\mu$, and due to our choice of $\Theta$ it is an
increasing function.  This phenomenon is more conspicuous for $\Theta_2$
than $\Theta_1$, since $\theta_{11}$ of the former is larger than that
of the latter.

One may also be interested in
$\rho_{\max} := \max_{0 \leq \mu \leq 1} \rho(\mu)$.
In order to estimate $\rho_{\max}$ we let $\mu \in \cbr{0.1, 0.2,
  \ldots, 0.9}$ and plotted the worst value of $\rho(\mu)$ as a function
of $n$ the number of nodes in Figure~\ref{fig:worst_mu}. 
In all cases
$\rho_{\max}$ was attained for $\mu=0.7$ or $\mu=0.9$.
It is empirically seen that the factor $\rho(\mu)$ increases 
as the number of nodes $n$ increases, but the speed of growth is
reasonably slow such that still the sampling of graphs with 
millions of nodes is feasible for any value of $\mu$.



\begin{figure}
  \begin{center}
    \begin{tikzpicture}[scale=0.51]
      \begin{axis}[title={$\Theta_1$}, 
        xlabel={number of nodes $(n)$},
        ylabel={estimated $\rho_{\max}$}, 
        legend pos=south east,
        domain=0:4000000,
        samples=4800
        ]
        
        \addplot+[color=blue, mark=x, thick]
        table[x=x,y=y,y error=errory]
        {cluster_mus_1.dat};
        \addplot[color=red, thick, dashed]{ln(x) * ln(x) * ln(x) * ln(x) /ln(2)/ln(2)/6000};
        \legend{Observed, $\rbr{\log_{2}(n)}^4$}
      \end{axis}
    \end{tikzpicture}
    \begin{tikzpicture}[scale=0.51]
      \begin{axis}[title={$\Theta_2$}, 
        xlabel={number of nodes $(n)$},
        legend pos=south east,
        domain=0:4000000,
        samples=5000
        ]
        
        \addplot+[color=blue, mark=x, thick]
        table[x=x,y=y,y error=errory]
        {cluster_mus_2.dat};
        \addplot[color=red, thick, dashed]{ln(x) * ln(x) * ln(x) *
          ln(x) /ln(2)/ln(2)/4600};
        \legend{Observed, $\rbr{\log_{2}(n)}^4$}
      \end{axis}
    \end{tikzpicture}
  \end{center}
  \caption{Estimated $\rho_{\max}$ as a function of the number of nodes
    for two different values of $\Theta$. }
  \label{fig:worst_mu}
\end{figure}
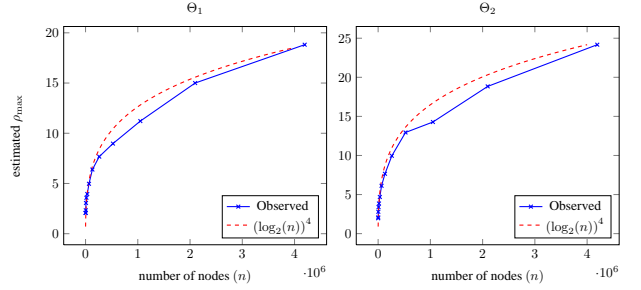

\subsection{Effect of $d$}
\label{sec:Effectd}

Now we study how the performance of our model varies as $d$ changes.
In particular we fix $n=2^{15}$ and vary $d$ to investigate its
effect on running time of our algorithm.  Figure~\ref{fig:effect_dim}
shows that there is no significant difference in the running time for $d
\leq \log_{2}(n)$. However, as we explained in
Section~\ref{sec:case_logn_neq_2d}, the running time of our algorithm
increases exponentially when $d > \log_{2}(n)$.
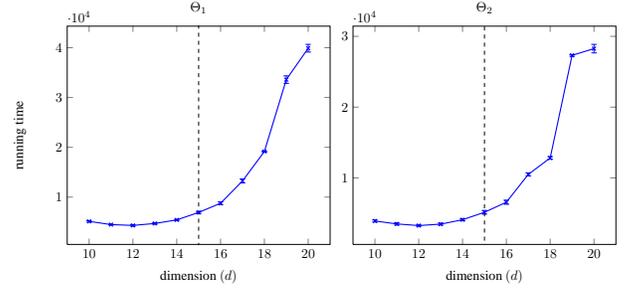
\begin{figure}
  \begin{center}
    \begin{tikzpicture}[scale=0.51]
      \begin{axis}[
        title={$\Theta_1$}, xlabel={dimension $(d)$},
        ylabel={running time}]

        \addplot+[
        color=blue, mark=x, thick,
        error bars/.cd,
        y dir=both, y explicit,
        error mark=-]
        table[x=x,y=y,y error=errory]
        {cluster_ds_1.dat};

        \draw[black,dashed] (axis cs:15,0) --
        (axis cs:15,100000);
      \end{axis}
    \end{tikzpicture}
    \begin{tikzpicture}[scale=0.51]
      \begin{axis}[
        title={$\Theta_2$}, xlabel={dimension $(d)$}
        ]

        \addplot+[
        color=blue, mark=x, thick,
        error bars/.cd,
        y dir=both, y explicit,
        error mark=-]
        table[x=x,y=y,y error=errory]
        {cluster_ds_2.dat};

        \draw[black,dashed] (axis cs:15,0) --
        (axis cs:15,100000);
      \end{axis}
    \end{tikzpicture}
  \end{center}
  \caption{The effect of dimension $d$ on the running time, where other
    parameters are fixed as $\mu = 0.5$ and $n = 2^{15}$.  $d = 15$
    (when $n = 2^d$) is highlighted by the dashed line.  }

  \label{fig:effect_dim}
\end{figure}

\section{Conclusion}
\label{sec:conclusion}

We introduced the first sub-quadratic algorithm for efficiently sampling
graphs from the MAGM. Under technical conditions, 
the expected running time
of our algorithm is $O\rbr{ \rbr{\log_{2}(n)}^{3} \abr{E}}$.  Our
algorithm is very scalable and is able to produce graphs with
approximately 8 million nodes in under 6 hours.  Even when the technical
conditions of our analysis are not met, our algorithm scales
favorably. We are currently working on rigorously proving the
performance guarantees for the case when $\mu \neq 0.5$.

Efficiently sampling MAGM graphs for the case when $d \geq \log_{2}(n)$
remains open. We are currently investigating how high-dimensional
similarity search techniques such as locality sensitive hashing (LSH) or
inverse indexing can be applied to this problem.




\newpage
\bibliographystyle{plainnat}
\bibliography{bibfile}
\newpage

.\\

\newpage

\appendix

\section{Technical Proofs}
\label{sec:TechnicalProofs}
\paragraph{Proof of Theorem~\ref{thm:optimal_partition}}
\label{sec:ProofTheorem}

Let $i' := \argmax_i \abr{Z_i}$. The attribute configuration
$\lambda_{i'}$ corresponding to node $i'$ appears at least $B$ times in
the set $\cbr{\lambda_1, \ldots, \lambda_{n}}$. By the pigeon-hole
principle any partition of the set $\cbr{1, \ldots, n}$ which contains
less than $B$ sets must have a set which contains two nodes $i$ and $j$
with attribute configuration $\lambda_i = \lambda_j =
\lambda_{i'}$. Therefore, the number of sets in the partition must be at
least $B$. Our partitioning scheme produces exactly $B$ sets, and is
therefore optimal.

\paragraph{Proof of Theorem~\label{thm:correctness}}
\label{sec:ProofTheorem-1}

By definition,
\begin{align}
  A_{i,j} = \sum_{1 \leq k,l \leq B} A_{i,j}^{(k,l)}.
\end{align}
To prove the theorem, we first show that $A_{i,j} = A_{\lambda_i,
  \lambda_j}'^{|Z_i|, |Z_j|}$.  This is straightforward from definition
\eqref{eq:unpermute}, since $D_1, \ldots, D_B$ is a partition of nodes
and thus $i \in D_k$, $j \in D_l$ for only $k = |Z_i|$ and $l = |Z_j|$.
This also implies
\begin{align}
  \PP\rbr{A_{ij} = 1 \mid \Thetat, \lambda_1, \ldots, \lambda_n}
  &=
  A_{\lambda_i, \lambda_j}'^{|Z_i|, |Z_j|} \nonumber \\
  &= P_{\lambda_i, \lambda_j} = Q_{i,j}, \nonumber
\end{align}
using \eqref{eq:kpgm_mag_con}.

To prove independence, we show that if $(i,j) \neq (i',j')$, then
$(\lambda_i, \lambda_j) \neq (\lambda_{i'}, \lambda_{j'})$ or $(|Z_i|,
|Z_j|) \neq (|Z'_i|, |Z'_j|)$.  Since we already showed $A_{i,j} =
A_{\lambda_i, \lambda_j}'^{|Z_i|, |Z_j|}$, this implies independence of
$A_{i,j}$ to other entries in $A$.

Now, suppose $(\lambda_i, \lambda_j) = (\lambda_{i'}, \lambda_{j'})$,
since if it does not hold there is nothing to prove.  Because $(i,j)
\neq (i',j')$ by assumption, at least one of $i \neq i'$ or $j \neq j'$
is true.  Without loss of generality, suppose $i \neq i'$.  Then, since
$\lambda_i = \lambda_{i'}$, $|Z_i| \neq |Z_{i'}|$ from definition of
$Z_i$.

\section{Chernoff Bound of Poisson Distribution}

\begin{theorem}
  \label{thm:chernoff}
  Let $X$ be the random variable which is distributed
  as Poisson distribution of parameter $\lambda$.
  Then,
  \begin{align}
    P(X \geq x) \leq \frac{e^{-\lambda} (e\lambda)^x}{x^x}.
    \label{eq:chernoff}
  \end{align}
\end{theorem}

\section{Upper bound of size of the partition when $n > 2^d$}

\label{sec:B_upperbound}

As a binomial distribution with finite mean in limit,
$Y_c$ is approximately distributed as
a Poisson distribution of parameter $\frac n {2^d}$.
To use Chernoff bound \eqref{eq:chernoff}
with $\lambda = \frac n {2^d}$, $x = 2^{t+1} \log_{2}(n)$,
$t = d'' - d$,
we first bound each term:
\begin{align}
  e^{-\lambda} = e^{-\frac{n}{2^d}} \leq e^{-2^t},
\end{align}
\begin{align}
  e^x = e^{2^{t+1} \log_{2}(n)} = n^{2^{t+1}},
\end{align}
\begin{align}
  \lambda^x = \left( \frac {n}{2^d} \right)^x
  \leq (2^{t+1})^{2^{t+1} \log_{2}(n)},
\end{align}
\begin{align}
  x^x = (2^{t+1} \log_{2}(n)) ^{2^{t+1} \log_{2}(n)}.
\end{align}
Then, plugging these into \eqref{eq:chernoff},
\begin{align}
  &\PP\rbr{ B = \max Y_c > 2^{t+1} \log_{2}(n) }
  \leq
  \sum_{c=1}^n \PP\rbr{ Y_c > 2^{t+1} \log_{2}(n) } \\
  &\;\;\leq
  n \cdot
  \frac{
    e^{-2^t}
    \cdot
    n^{2^{t+1}}
    \cdot
    (2^{t+1})^{2^{t+1} \log_{2}(n)}
  }
  {(2^{t+1} \log_{2}(n)) ^{2^{t+1} \log_{2}(n)}}
  .
\end{align}
By taking log,
\begin{align}
  \log &\PP\rbr{ B = \max Y_c > 2^{t+1} \log_{2}(n) }
  \leq
  -2^t + 2^{t+1} \log_{2}(n) \nonumber \\
  &+ (2^{t+1} \log_{2}(n)) \log_{2} 2^{t+1}
  - 2^{t+1} \log_{2}(n)
  \log_{2} (2^{t+1} \log_{2}(n)) \nonumber \\
  &=
  -2^t
  + 2^{t+1} \log_{2}(n)
  (1 + \log_{2} 2^{t+1}
  - \log_{2}(2^{t+1} \log n))
  ,
\end{align}
and this goes to $-\infty$ as $n \rightarrow \infty$.
Therefore, $\PP[ B = \max Y_c > 2^{t+1} \log_{2}(n)] \rightarrow 0$.

\end{document}